\documentclass{article}

\usepackage[final]{neurips_2019}

\usepackage{rotating}
\usepackage{times}

\usepackage[utf8]{inputenc} 
\usepackage[T1]{fontenc}    
\usepackage{hyperref}       
\usepackage{url}            
\usepackage{booktabs}       
\usepackage{amsfonts}       
\usepackage{nicefrac}       
\usepackage{microtype}      

\usepackage{caption}
\usepackage{subcaption}
\usepackage{tikz}
\usepackage{xcolor}
\usetikzlibrary{arrows}
\usepackage{threeparttable}
\usepackage{multirow}
\usepackage{cancel}
\usepackage{tabularx}

\usepackage{paralist}
\usepackage{todonotes}

\usepackage{times}
\usepackage{graphicx}
\usepackage{caption}
\usepackage{subcaption}
\usepackage{natbib}
\usepackage{algorithm}
\usepackage{algorithmic}
\usepackage{hyperref}

\usepackage{url}            
\usepackage{booktabs}       
\usepackage{amsfonts}       
\usepackage{nicefrac}       
\usepackage{microtype}      
\usepackage{color}
\usepackage{amsmath}
\usepackage{bm}
\usepackage{amsthm}
\usepackage{times}
\usepackage{amssymb}
\usepackage{xr}
\usepackage{lscape}
\usepackage{dsfont}
\usepackage{bbm}

\newtheorem{lemma}{Lemma}

\def\nn{\nonumber}

\def\defeq{\triangleq}
\def\mc{\mathcal}

\def\LM{\mathrm{LM}}

\allowdisplaybreaks

\title{From Caesar Cipher to Unsupervised Learning:\nn\\  A New Method for Classifier Parameter Estimation}

\author{%
  Yu Liu \thanks{Yu~Liu and Li~Deng are with AI Research of Citadel LLC, Seattle, WA. Jianshu~Chen is with Tencent AI Lab, Belleve, WA. And~Chang~Wen~Chen is with Department of Computer Science and Engineering, State University of New York at Buffalo. The research reported in this paper is based on the Ph.D. thesis work of the first author and on the research carried out at Microsoft Research while he was a research intern there.} \\
  AI Research\\
  Citadel LLC\\
   \And
   Li Deng \\
   AI Research\\
   Citadel LLC \\
   \And
   Jianshu Chen \\
   AI Lab \\
   Tencent\\
   \And
   Chang Wen Chen \\
   Computer Science and Engineering\\
   University at Buffalo, SUNY \\
}

\begin{document}
\maketitle

\begin{abstract}
Many important classification problems, such as object classification, speech recognition, and machine translation, have been tackled by the supervised learning paradigm in the past, where training corpora of parallel input-output pairs are required with high cost. To remove the need for the parallel training corpora has practical significance for real-world applications, and it is one of the main goals of \emph{unsupervised learning} aimed for pattern classification with label-free data sources for training. Recently, encouraging progress in unsupervised learning for solving such classification problems has been made and the nature of the challenges has been clarified. In this article, we review this progress and disseminate a class of promising new methods to facilitate understanding the methods for machine learning researchers. In particular, we emphasize the key information that enables the success of unsupervised learning --- the sequential statistics as the distributional prior in the labels as represented by language models. Exploitation of such sequential statistics makes it possible to learn parameters of classifiers without the need to pair input data and output labels. The idea of exploiting sequential statistics for our unsupervised learning is inspired by an ancient encryption technique called Caesar Cipher. \nn\\In this paper, we first introduce the concept of Caesar Cipher and its decryption, which motivated the construction of the novel loss function for unsupervised learning we use throughout the paper. The loss function serves as the basis for a novel learning algorithm that forms the core content of this paper. Then we use a simple but representative binary classification task as an example to derive and describe the unsupervised learning algorithm in a step-by-step, easy-to-understand fashion. We include two cases, one with Bigram language model as the sequential statistics for use in unsupervised parameter estimation, and another with a simpler Unigram language model. For both cases, detailed derivation steps for the learning algorithm are included in the paper. Further, a summary table of computational steps in executing the unsupervised learning algorithm for learning binary classifiers is provided. In the summary, we also include a comparison between the computational steps for the use of Unigran and Bigram language models, respectively.
\end{abstract}

\section{Introduction and Background}


The long-standing success of supervised learning, especially with the use of deep neural networks on big data since 2010 or so, has brought to machine learning practitioners a tremendous opportunity to solve a wide range of challenging problems in pattern recognition and artificial intelligence \citep{LeCun2015,Bahdanau2014Neural,Deng2013,DNN_SPMagazine12,Abdel-Hamid2014}. Classification or input-to-output prediction is one of the most fundamental problems in machine learning, and it has typically been formulated in the setting of supervised learning, where the learning algorithms require both input data and the corresponding output labels in pairs (i.e. parallel data) to learn the input-to-output prediction with the model. 

However, the acquisition of output labels to pair with the input data for supervised learning is often very expensive, especially for large-scale deep-learning systems that generally require big data for the training. It is not uncommon that thousands of hours of human labor are devoted to manually label a considerable size of a dataset for a specific task before the training can even start to take place. For example, over ten million images had been hand-annotated to build the ImageNet dataset for the image classification task \citep{ImageNet}. Each image was annotated with reasonably unambiguous tags to indicate what notable objects are included. For large-scale speech recognition, even greater amounts of annotated data are required to achieve the desired high accuracy \citep{YuDeng2015,DengHinton2013}. Collecting such labeled datasets incurs high cost in human labor and thus prevents the modern deep learning systems based on the current supervised-learning paradigm from scaling out to even greater data in training. It is thus highly desirable to invent methods that would enable the training of large-scale classification models without using labels that match the data on the sample-by-sample basis.
 
This calls for \emph{Unsupervised learning} in the context of predicting output labels from input data. While unsupervised learning has been studied for decades in the literature, none of common existing unsupervised learning methods can handle the task described above, For example, the most common unsupervised learning methods of clustering analysis include K-means and spectral clustering \citep{Ng2002}. They learn to exploit the structure of input data and extracts grouping patterns by exploring the inter-relation between data samples. Clustering analysis alone cannot address the problem of predicting output labels from input data. 

There are many classes of unsupervised learning methods, i.e. learning without labels, developed in recent years. In one class, structure of the data in terms of density functions is modeled (no labels needed), either explicitly \citep{Oord2016,Oord2016a} or implicitly \citep{goodfellow2014generative}, giving a generative model of the data when one can draw samples and enabling one to examine what the model has learned or otherwise \citep{eslami2018neural,Jakab2018unsupervised,Gupta2018learning}. The generative models can also be used to imagine possible scenarios for model-based reinforcement learning and control. 

In another type, the structure or inherent representations of the data are modeled not by deriving its density functions but by exploiting and predicting the temporal sequence (future and history) of the data \citep{Finn2016unsupervised,Mikolov2013}. A related class of unsupervised representation learning is to model the hierarchical internal representations of input data not via sequence prediction but through modeling latent representations including graph representations \citep{vincent2010stacked,DengHinton2010,kingma2013auto,Yang2018glomo}. 

The next class of unsupervised learning methods targets the task of classification by exploiting the structured bidirectional mapping between input and output, both of the same modality such as images or texts. These methods take advantage of the similarity of information rates in input and output, enabling both forward and inverse mappings and accomplishing impressive tasks of image-to-image translation \citep{Mejjati2018unsupervised,Kazemi2018unsupervised,Zhu2017} and of language-to-language translation  \citep{Artetxe2017,Ranzato2017,Yang2018unsupervised}. However, when the input and output are of substantially different information rate or of different modalities --- e.g. mapping from image to text or from speech to text --- these methods would not work because the forward and inverse mappings can no longer take the same functional form as required by these methods. 

Related to the above class of methods but without constraining input and output to be of the same modality, another popular class of unsupervised learning for classification is to perform pre-training without labels \citep{Ramachandran_2017}. Subsequently fine-tuning the pre-trained parameters is carried out with limited labels. This class of methods has accounted for the early success of deep learning in speech recognition \citep{YuDeng2010,dahl2012context,DengYubook2014,YuDeng2015}. Very recently since October of 2018, they have also shown large success in natural language processing \citep{BERT2018,OpenAI-GPT2}.

However, to solve classification problems helped by the stage of unsupervised pre-training still requires labeled data in the fine-tuning stage. When labeled data is not so costly, this problem is not very serious. To be completely label free, it is desirable to carry out end-to-end learning via direct optimization. A class of new methods with this motivation have been developed in recent years, and are the focus of this paper.

We in this paper introduce a most interesting class of unsupervised learning methods aimed to address the problem of classifying output labels (typically of low information rate) from input data (typically of high information rate) in parametric models via direct optimization, where, unlike all previous methods, the training of the model parameters does not require explicit labels for each training sample. The essence of the methods is a novel objective function for the training, together with an uncommon technique of optimization, called stochastic primal-dual gradient (SPDG) \citep{Convex,Chen2016}. These methods have recently been developed and shown to be effective in an unsupervised optical character recognition task \citep{NIPS2017}. Since the material in \citep{NIPS2017} is general and dense in mathematical treatment and written mainly for machine learning experts, this article is aimed to disseminate detail of the unsupervised learning methods and to serve as a tutorial for signal processing readers who are less experienced in machine learning. The article will first intuitively introduce a specific method, starting from an inspiration of decryption of Caesar Cipher for motivating the construction of the objective function. We will then formulate the problem by connecting this Caesar Cipher problem to a simple binary classification problem whose model training would not require data-label pairs. The deciphering procedure in Caesar Cipher contains three critical steps, which can be closely linked to the analogous three steps in the SPDG optimization method. 

The remaining content of this article is organized as follows. Section \ref{SECTION:Decryption} introduces the procedure of decrypting Caesar Cipher, which consists of three critical steps. Next, in Section \ref{SECTION:Formulation}, we connect the binary classification problem to this Caesar Cipher problem, and formulate the problem as an optimization problem for a novel objective function that leverages a prior distribution of labels or a ``(unigram) language'' model. In Section \ref{SECTION:AnEquivalent}, We show that this objective function for unsupervised learning is intractable with common stochastic gradient descent (SGD) techniques. Thus the optimization problem is transformed to an mathematically equivalent, saddle-point, problem. The saddle-point problem can be approached by a variant method of SPDG, called primal-dual method. In Section \ref{SECTION:Extension}, we extend the formulation and optimization of our unsupervised learning problem from the use of simple unigram language model to a more complex bigram language model. In Section \ref{SECTION:Experiment}, we show our experimental results, where we create a synthetic dataset and apply the unsupervised learning method with the use of unigram and bigram language models, respectively. We conclude and summarize the article in Section \ref{SECTION:Conclusions}.

 


\section{Decryption of Caesar Cipher --- A Motivating Example of Unsupervised Learning}
\label{SECTION:Decryption}

\subsection{Caesar Cipher}

The Caesar cipher, also known as shift cipher, is one of the simplest forms of encryption. It has been first used by Julius Caesar, the well-known Roman general who led to the rise of the Roman Empire. And the encryption technique was named after him. Caesar Cipher is a substitution cipher where each letter in the plain text, or the \emph{original message}, is replaced with another letter corresponding to a certain number of letters up or down in the alphabet. For example, in Figure \ref{fig:caesarcipher}, the message in the plain text is ``THE QUICK BROWN FOX JUMPS OVER THE LAZY DOG''. With a shift of 3 letters up, the encrypted message becomes ``QEB NRFZH YOLTK CLU GRJMP LSBO IXWV ALD''.  

\begin{figure}[h]
		\centering
   		\includegraphics[width = .8\linewidth]{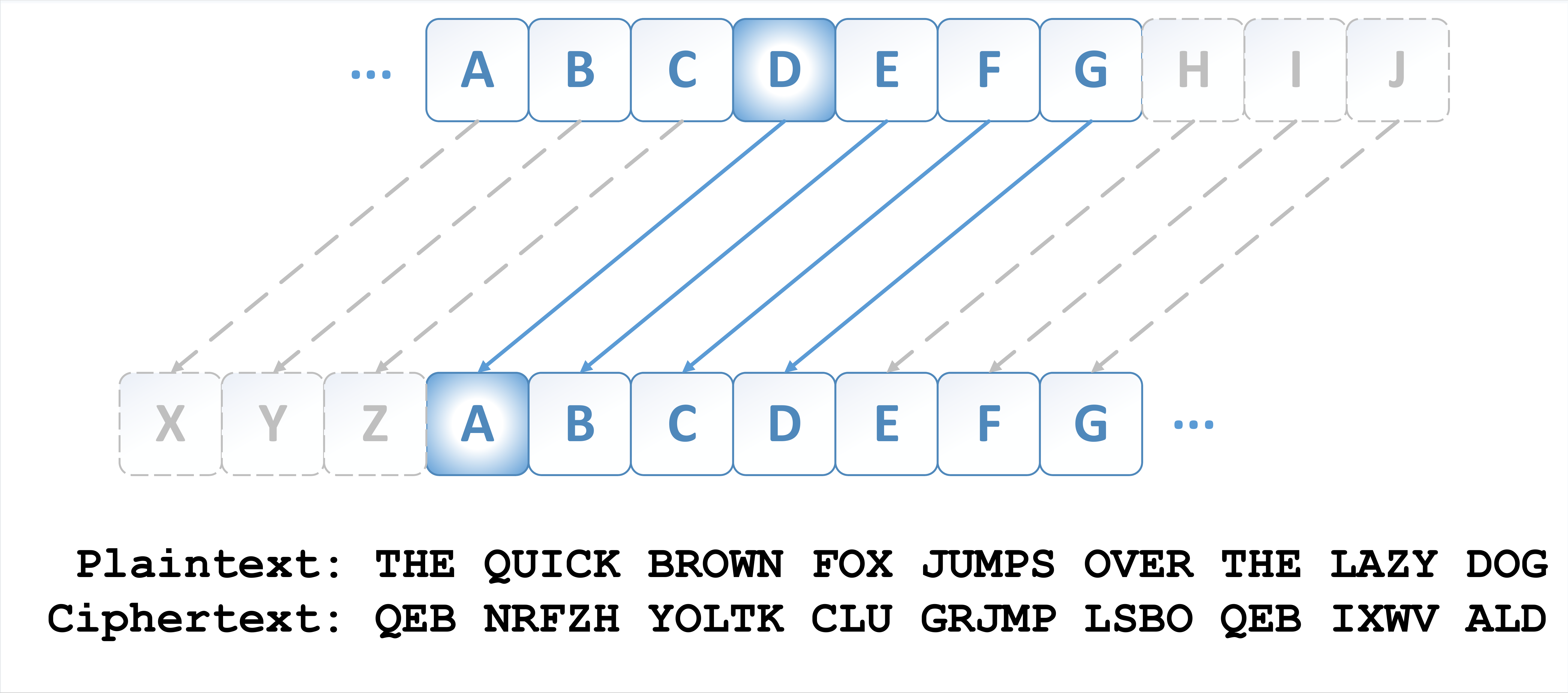}
   		\caption{Illustration of Caesar Cipher with shift-3 up}
   		\label{fig:caesarcipher}
	\end{figure}
    
\subsection{Decryption of Caesar Cipher}


To decrypt Caesar Cipher with unknown shift, we need three steps: 1) acquiring standard English statistics as a ``prior''; 2) calculating the same statistics of the encrypted message, and 3) matching the two statistics with a trial-and-error scheme to recover the shift. Here we assume that the encrypted message is large enough that the number of letters has statistical significance. The process is illustrated in Figure \ref{fig:letterfrequency} and explained in details as follows \cite{OnlineTool}:
   
\textbf{Step 1: Prior standard English statistics} To discover the correspondence between the encrypted letters and standard English letters, we first need to know precisely how the encrypted letters look like in the standard alphabet. To this end, for example, one can take a large bulk of standard English corpus (e.g.  newspapers) and calculate the frequency or histogram of each of 26 English letters, as shown in lower part of the Figure \ref{fig:letterfrequency}. 

\textbf{Step 2: Encrypted message statistics} Similarly, we can determine the same statistics on our encrypted original message. This can be accomplished by counting the frequency of each letters in the encrypted message. We may give the histogram that looks like the one on the upper part of Figure \ref{fig:letterfrequency} (after a shift).

\textbf{Step 3: Matching the two statistics} With the two histogram charts side by side, one can easily try out in a brute-force way all of 25 shifts and find the best match, as shown in Figure \ref{fig:letterfrequency}. After determining the shift, one can obtain the shifted alphabet and recover the original message. 

\begin{figure}[h]
		\centering
   		\includegraphics[width = .8\linewidth]{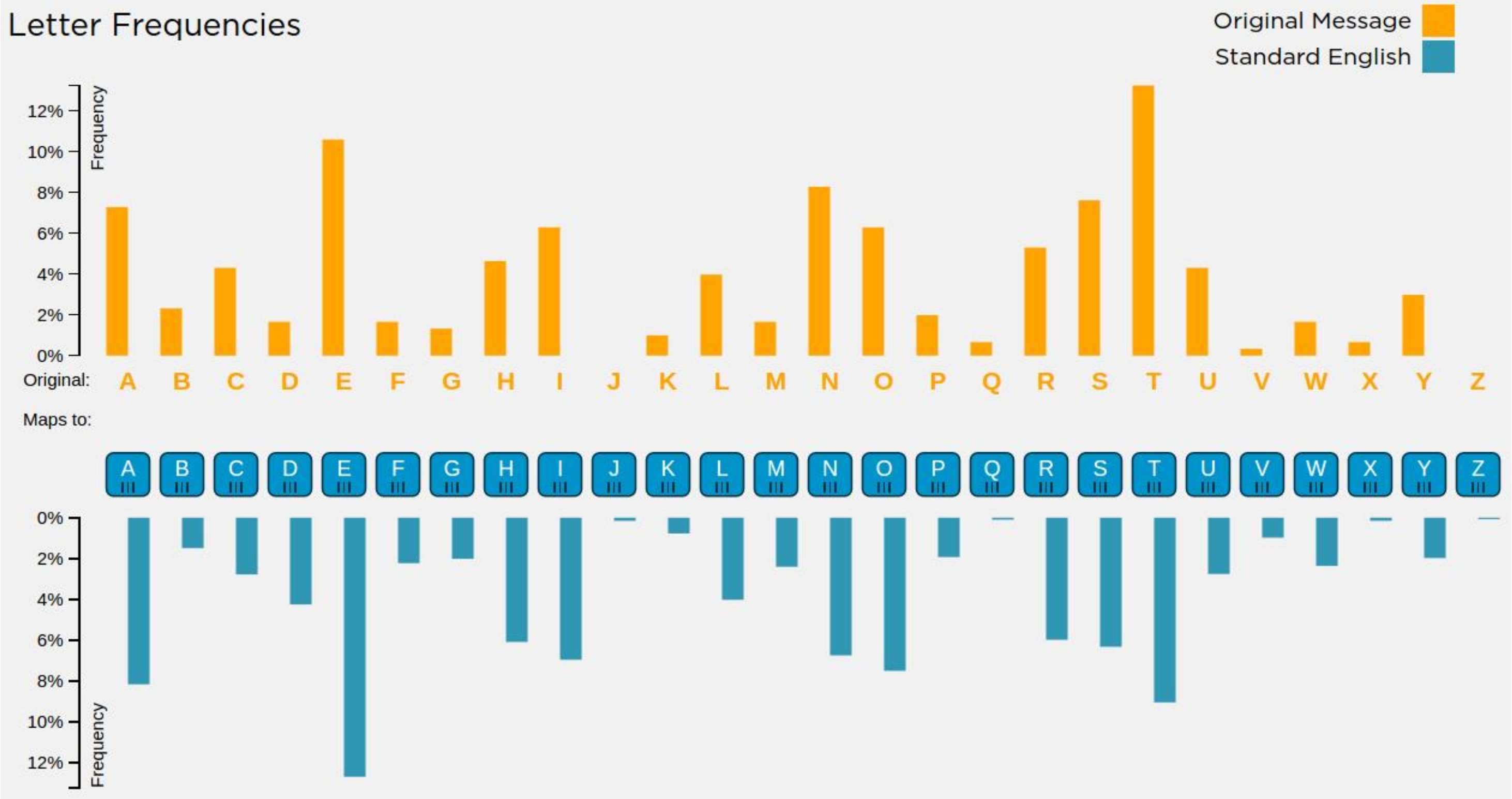}
   		\caption{Letter frequency analysis and match between standard English and original message. [Screenshot from Caesar Cipher's online tool \citep{OnlineTool}]}
   		\label{fig:letterfrequency}
	\end{figure}

\section{Formulation of Unsupervised Learning}
\label{SECTION:Formulation}

The above motivating example of deciphering the original messages (input) from the encrypted message by Caesar Cipher (input) highlighted the importance of exploiting the prior output statistics of English letters. Note that the three-step procedure from input to output described in Section II does not require the kind of supervised learning with costly input-output pairs in the training set. We now formulate the unsupervised learning problem without such a requirement as well.

\subsection{Formulation of the Unsupervised Learning Problem}

Our unsupervised learning problem here is formulated in the simplest setting of binary classification where no input-output pairs are needed in the classifier training. The input $\mc{X}=(\mathbf{x_1}, \dots ,\mathbf{x_T})$ is a sequence of 2-dimensional vectors, where $\mathbf{x_t}=(x_t^a,x_t^b) \in \mathbb{R}^{2 \times 1}$ is a 2-dimensional vector. The ground-truth labels $\mc{Y}=(y_1, \dotsm y_T)$ are the corresponding sequence, where each $y_t$ denotes the class that $x_t$ corresponds to. Note that the ground-truth labels are not used in the training. Here, analog to Caesar Cipher, the input sequence $\mc{X}$ acts as the encrypted message, and $\mc{Y}$ is the original message to recover. 

Importantly, as any original English message is subject to English language rules, we also ask that $y_t$ be subject to a given ``rule''. In this context, we call this rule as ``language model'' in the form of N-gram, i.e. the statistics of $y_t \sim p(y_N|y_{N-1},\dots ,y_1)$.  In case of $N=1$, the language model reduces to the basic frequency distribution, i.e. Unigram, just like the statistics used in the example of Caesar Cipher. In case of $N>1$, the language model describes the sequential pattern by predicting the next letter in the form of a $(N-1)$-th order Markov model, e.g. $N=2$ the Bigram model. We will formulate unsupervised learning in detail assuming the Unigram model in this section, and walk through the Bigram case in Section V. 

Our problem here is very similar to that of decryption of Caesar Cipher; i.e. given only the N-gram statistics (analogous to the Prior standard English statistics) and a sequence $(\mathbf{x_1}, \dots ,\mathbf{x_T})$, which is analogous to the original message) as the input, we desire to predict the true label sequence $(y_1, \dotsm y_T)$, which is analogous to the decrypted message).

\subsection{An Objective Function for Unsupervised Learning}

Since the problem discussed here for pedagogical purposes is binary classification, we use the most common model, the log-linear model, as our classifier. In the log-linear model, the output of the classifier is a posterior probability. Recalling that the input is two-dimensional $\mathbf{x_t}=(x_t^a,x_t^b)$, we define the model that has two learnable parameters denoted as $\theta=\{ w^a, w^b \}$:
	\begin{align}
		&p_{\theta,t}(0)
        \defeq
        p_{\theta}(y_t=0 | \mathbf{x_t}) 
        = 
        \frac{ e^{\gamma w^a x^a_t} }{e^{\gamma w^a x^a_t} + e^{\gamma w^b x^b_t}}
        \nn\\
        &p_{\theta,t}(1)
        \defeq
        p_{\theta}(y_t=1 | \mathbf{x_t})
        =
        \frac{ e^{\gamma w^b x^b_t} }{e^{\gamma w^a x^a_t} + e^{\gamma w^b x^b_t}}
        \label{Equ:logLinearModel}
	\end{align}
where $\gamma$ is a constant (e.g. we fix $\gamma=10$ in our experiments described in Section VI). These formulas calculate the probability of producing $y_t$ when receiving input $\mathbf{x_t}$. For example, if $p_{\theta,t}(0) \geq 0.5$, we predict label $\bar{y_t}=0$; otherwise we predict label $\bar{y_t}=1$. Note that from Eqn.(\ref{Equ:logLinearModel}), we always have $p_{\theta,t}(0)+p_{\theta,t}(1)=1$.

We approach the binary classification problem with unsupervised learning in three steps, analogous to decrypting Caesar Cipher described in Section III-B. The corresponding three main steps are: 1) Determine language model statistics as the prior on labels $\mc{Y}$; 2) Calculate the statistics on the classifier output, and 3) Match the two statistics with which a cost function is determined. Next we carry out this procedure step-by-step to create the cost function.

\textbf{Step 1: Prior language model statistics.} Since the dataset we use is synthetic data, i.e. both $\mc{X}$ and $\mc{Y}$ are known, we can easily obtain the statistics of a Unigram language model by calculating the probabilities below with all $ y_t \in \mc{Y} $ given in the dataset:
\begin{align}
		\mathbf{P_{LM}} 
        \defeq
        \begin{bmatrix}
        p(y_t=0),~~p(y_t=1)
        \end{bmatrix}
        =
        \begin{bmatrix}
        p_{0},   ~~
        p_{1}
        \end{bmatrix}
        \label{Equ:prior}
	\end{align}
Note that $p_0,p_1 \geq 0$ are constants with $p_0 + p_1 = 0$. 

\textbf{Step 2: Classifier output statistics.} The classifier outputs from Eqn.(\ref{Equ:logLinearModel}) directly determine the probability of data being assigned with each label. The Unigram (or histogram) statistics can be easily calculated by
\begin{align}
		\mathbf{\overline{P_{LM}}} (\theta) =
        \begin{bmatrix}
        \frac{1}{T}\sum_{t=1}^T p_{\theta,t}(0),~~\frac{1}{T} \sum_{t=1}^T p_{\theta,t}(1)
        \end{bmatrix}^T
        \label{Equ:classifierstat}
	\end{align}
where $\mathbf{\overline{P_{LM}}} (\theta)$ stands for the estimated language model from classifier which is parameterized by $\theta$.

\textbf{Step 3: Matching the two statistics.}  In the same spirit of decryption of Caesar Cipher, we need to compare and then match the two statistics $\mathbf{P_{LM}} $ and $\mathbf{\overline{P_{LM}}} (\theta)$ in order to find the best decipher. As suggested in \citep{Chen2016}, the difference between two distributions can be measured by Kullback-Leibler (KL) divergence. The KL divergence of the true prior distribution $\mathbf{{P_{LM}}}$ and the estimated distribution $\mathbf{\overline{P_{LM}}}$ is:
\begin{align}
	D_{KL}(\mathbf{{P_{LM}}}||\mathbf{P_{LM}})
	=
    \mathbf{P_{LM}} \ln{{\mathbf{P_{LM}}}^{T} \oslash {\mathbf{\overline{P_{LM}}(\theta)}} }
    =
    -\mathbf{P_{LM}}  \ln{\mathbf{\overline{P_{LM}}} (\theta)} + \mathbf{P_{LM}}  \ln{\mathbf{{P_{LM}}^{T}}} 
    \label{Equ:KLDivergence}
	\end{align}
where the sign $\oslash$ denotes element-wise division. We minimize $D_{KL}$ in order to learn the model parameters $\theta$ to the effect that the estimated statistics become as close to the prior statistics as possible. Note that in $D_{KL}$ of Eqn. (\ref{Equ:KLDivergence}), the second term $\mathbf{P_{LM}},  \ln{\mathbf{{P_{LM}}^{T}}}, $ is a constant and can be ignored in optimization. Thus, our objective function to be minimized becomes:
	\begin{align}
	\mc{J}(\theta)
	=
    -\mathbf{P_{LM}}  \ln{\mathbf{\overline{P_{LM}}} (\theta)}
    =
	  -p_{0} 	\ln{\frac{1}{T} \sum_{t=1}^{T}{p_{t,\theta}(0)}}  -  p_{1} 	\ln{\frac{1}{T} \sum_{t=1}^{T}{p_{t,\theta}(1)}}  
    \label{Equ:CostUnigram}
	\end{align}
    
In summary, we formulate our problem of unsupervised learning in this pedagogical example as an optimization problem of finding the best classifier parameters $\theta^*$ according to
\begin{align}
	\theta^* = \arg \min_\theta \mc{J}(\theta)
    \label{Equ:minCost}
	\end{align}

\subsection{Issues of SGD-based Optimization}

However, for the objective function of Eqn.(\ref{Equ:CostUnigram}), it is not feasible to use SGD-based optimizer to solve the above minimization problem above. This is due to two reasons: the high cost of computing the gradient and the highly non-convex profile of the objective function with respect to the parameters.

\textbf{High cost of computing the gradient.} The SGD-based optimizers, such as AdaGrad \citep{AdaGrad} and ADAM \citep{Kingma2014}, require that the sum over training samples be outside of the logarithmic loss \citep{Maximum}. However, in the above loss function, the sum over samples $t$, i.e. $\frac{1}{T} \sum_{t=1}^{T}$, is inside the logarithmic loss ($\ln$). This makes a very large difference from traditional neural network error minimization problems solved by SGD. The special form of the objective function in Eqn. (\ref{Equ:CostUnigram}) prevents us from applying SGD to minimize the cost.

\textbf{Highly non-convex profile.} Even worse, the cost function is highly non-convex, and hard to converge to the optimum. Figure \ref{Fig:Surface2DUnigram} illustrates this non-convex surface of the cost function $\mc{J}(\theta)$ in in Eqn. (\ref{Equ:CostUnigram}). To obtain this profile surface, we take two steps. First, we use both data $\mc{X}$ and true labels $\mc{Y}$ to find the global optimum $\theta^0=(w^{a0},w^{b0})$ (the red points in the figures) by solving the \textbf{supervised} learning problem
$$
	\theta^0 = \arg\max_\theta \sum_{t=1}^{T} \ln{p_{t,\theta}(y_t|x_t)}.
    \label{Equ:Supervise}
$$
This is an ordinary maximum posterior optimization problem with the sum over training sameples lying outside of the logarithm. Hence we solve it by plain SGD and plot with red dot in Figure \ref{Fig:Surface2DUnigram}. In the second step, we plot the two-dimensional function around $w^{a0},w^{b0}$, i.e. $\mc{J}(w^{a0} + \lambda_1, w^{b0} + \lambda_2)$ with respect to $\lambda_1, \lambda_2 \in \mathbb{R}$. We can observe the surface of the cost function $\mc{J}(\theta)$ and find that it is highly non-convex. There are many local optimal solutions and there is a high barrier that makes it difficult for naive optimization algorithms to reach the global optimum.

To overcome the above two difficulties, we can transform the optimization problem for the cost function in Eqn. (\ref{Equ:CostUnigram}) into an equivalent one, called \emph{saddle point problem}, where the cost function has a more desirable profile surface for gradient-based methods. One such technique is introduced in Section IV next.

\begin{figure*}[t!]
	\centering
	\begin{subfigure}[t]{0.3\textwidth}
		\includegraphics[width=\textwidth]{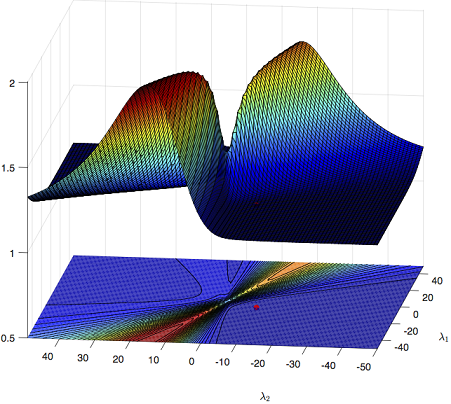}
		\caption{}
	\end{subfigure}	
	\quad
	\begin{subfigure}[t]{0.3\textwidth}
		\includegraphics[width=\textwidth]{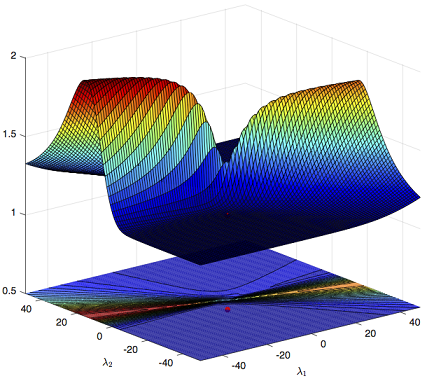}
		\caption{}
	\end{subfigure}
	\quad	
	\begin{subfigure}[t]{0.3\textwidth}
		\includegraphics[width=\textwidth]{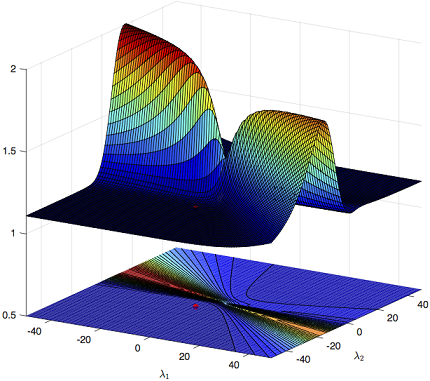}
		\caption{}
	\end{subfigure}	
	\quad
	\begin{subfigure}[t]{0.3\textwidth}
		\includegraphics[width=\textwidth]{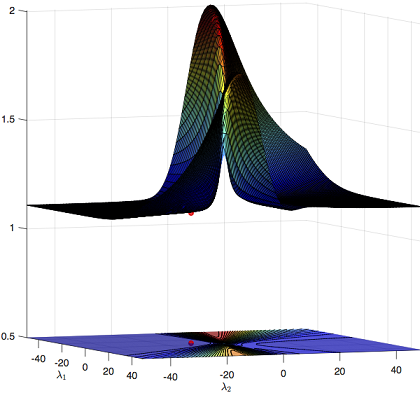}
		\caption{}
	\end{subfigure}	
	\quad
	\begin{subfigure}[t]{0.3\textwidth}
		\includegraphics[width=\textwidth]{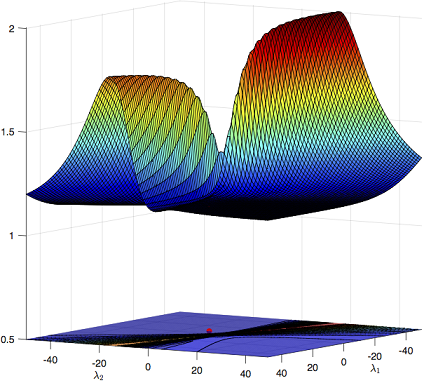}
		\caption{}
	\end{subfigure}
	\quad	
	\begin{subfigure}[t]{0.3\textwidth}
		\includegraphics[width=\textwidth]{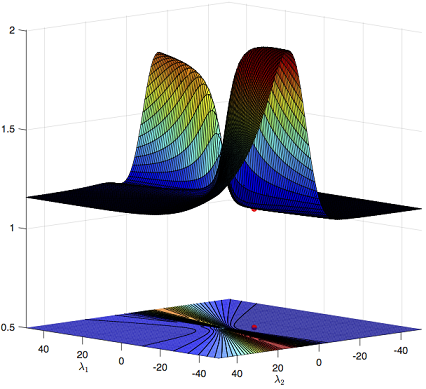}
		\caption{}
	\end{subfigure}	
	\caption{$\mc{J}(\theta)$ profile in the primal domain. The six sub-figures show the same profile from six different angles, spinning clock-wise from (a)-(f). The red dots indicate the global minimum.
	}
	\label{Fig:Surface2DUnigram}
\end{figure*}

\section{An Equivalent Saddle Point Problem \& Its Solution}
\label{SECTION:AnEquivalent}
\subsection{Problem Formulation}

Let us first introduce a lemma that will be used in the later description of the saddle point problem.

  \begin{lemma}
  \textbf{For any scaler $u>0$, the equation below holds:}
  \begin{align*}
          -\ln{u}
          =
          \max_{v<0}{(uv+\ln{(-v)})}.
      \end{align*}
      
      \label{Equ:lemmaConj}
  \end{lemma}

\begin{proof}
\vspace{2mm}
\itshape
We use convex conjugate functions to prove this equation. For a given convex function $f(u)$, its convex conjugate function  $f^* (v)$ is defined in \citep{Convex} as:
	\begin{align}
    	f^* (v)
        \defeq
        \sup_u{(v^Tu-f(u))}.
        \label{Equ:conjugate}
    \end{align}
Furthermore, the convex function $f(u)$ can be related to its convex conjugate function $f^* (v)$ as following relation:
	\begin{align}
    	f (u)
        =
        \sup_v{(u^T v-f^*(v))}.
        \label{Equ:relationToConjugate}
    \end{align}
We now consider a convex function $f(u)=-\ln{u}$ where $u$ is a scalar and $u>0$. From Eqn. (\ref{Eqn:conjugate}), the conjugate function for $f(u)$ can be obtained:
	\begin{align*}
    	f^* (v)|_{f(u)=-\ln{u}}
        =
        \sup_u{(uv+\ln{u})}
        =
        \max_u{(uv+\ln{u})}.
    \end{align*}
Let ${\partial{(uv+\ln{u})}}/{\partial{u}}=v+{1}/{u}=0$ and we can obtain the maximum point for $u$ is at $u=-{1}/{v}$. Substituting it in above formula, we can solve to get
	\begin{align}
    	f^* (v)|_{f(u)=-\ln{u}}
        =
        (uv+\ln{u})|_{u=-\frac{1}{v}}
        =
        -1-\ln{v}
        \label{Equ:conjugateOfLn}
    \end{align}
with $v<0$ is a scaler. We then replace the $f^* (v)$ in the right side of Formula (\ref{Equ:relationToConjugate}) with Formula (\ref{Equ:conjugateOfLn}), and replace the $f(u)$ in left side with $f(u)=-\ln{u}$. Finally, for any scaler $u>0$, we obtan:
	\begin{align*}
    	-\ln{u}
        =
        \max_{v<0}{(uv+\ln{(-v)})}
    \end{align*} 
\vspace{2mm}
\end{proof}

Examining Lemma \ref{Equ:lemmaConj}, we find $\ln{u}$ on the left side of the equation is transformed into an maximum over a different variable $v$ of a formula where $u$ comes outside the logarithm. Thus, we use this Lemma to rewrite the logarithm part in the cost function in Eqn. (\ref{Equ:CostUnigram}):
	\begin{align}
	\min_{\theta}&{\mc{J}(\theta)} 
	=
	\min_{\theta} {\left \{ -p_{0} 	\ln{\frac{1}{T} \sum_{t=1}^{T}{p_{t,\theta}(0)}}  -  p_{1} 	\ln{\frac{1}{T} \sum_{t=1}^{T}{p_{t,\theta}(1)}}  \right \}}
    \nn\\
    =&
    \min_{\theta}\Bigg \{  p_{0} 	\max_{v_{0}<0}{\left [  v_{0} \frac{1}{T} \sum_{t=1}^{T} p_{t,\theta}(0)+\ln{(-v_{0})}  \right ]}  
    + p_{1}  \max_{v_{1}<0}{\left [  v_{1} \frac{1}{T} \sum_{t=1}^{T} p_{t,\theta}(1)+\ln{(-v_{1})}  \right ]} 
\Bigg \} 
    \nn\\
    =&
    \min_{\theta} \max_{v_{0},v_{1}<0}  \Bigg \{  p_{0} v_{0} \frac{1}{T} \sum_{t=1}^{T} p_{t,\theta}(0) + p_{0} \ln{(-v_{0})} 
    + p_{1} v_{1} \frac{1}{T} \sum_{t=1}^{T} p_{t,\theta}(1) + p_{1} \ln{(-v_{1})}
    \Bigg \}  
    \nn\\
    =& 
    \min_{\theta} \max_{v_{0},v_{1}<0 }  \Bigg \{  \frac{1}{T} \sum_{t=1}^{T} \left [ 
    p_{0} v_{0} p_{t,\theta}(0) + p_{1} v_{1} p_{t,\theta}(1) \right ] 
    + p_{0} \ln{(-v_{0})} + p_{1} \ln{(-v_{1})}  \Bigg \}
    \label{Equ:final}
	\end{align}

Let us now define the new cost function $\mc{L} (\theta,V)$ to be the quantity inside the $\min \max$ in Equation \ref{Equ:final} above:
\begin{align}
	{\mc{L}(\theta,V) }
    =
     \frac{1}{T} \sum_{t=1}^{T} \left [ 
    p_{0} v_{0} p_{t,\theta}(0) + p_{1} v_{1} p_{t,\theta}(1) \right ] 
    + p_{0} \ln{(-v_{0})} + p_{1} \ln{(-v_{1})} 
    \label{Equ:finalNew}
	\end{align}
where $V=\{ v_0, v_1 \}$. 
With the use of matrix form 
	$\mathbf{V}=
        \begin{bmatrix}
        v_{0}, v_{1}
        \end{bmatrix}
	$,
as well as Equations (\ref{Equ:prior}) and (\ref{Equ:classifierstat}), we rewrite the new cost function into a matrix form (where $\odot$ stands for pairwise multiplication between two matrices):
\begin{align}
	\min_{\theta}{ \max_{\mathbf{V}<\mathbf{0}}  {\mc{L}(\theta,\mathbf{V}) } }
    =
    \min_{\mathbf{\theta}} \max_{\mathbf{V}<\mathbf{0}}  \bigg \{ 
      [\mathbf{P_{LM}} \odot  \mathbf{V}]  \mathbf{\overline{P_{LM}}} (\theta)  
    +\mathbf{P_{LM}} \ln{(-\mathbf{V})}^T 
    \bigg \}  
    \label{Equ:finalMatrix}
	\end{align}
    
Let us provide interpretation for the new cost function $\mc{L}(\theta,\mathbf{V})$ in Eqn. (\ref{Equ:finalMatrix}), contrasting the earlier form of $\mc{J}(\theta)$ in Eqn. (\ref{Equ:CostUnigram}). In the new form, parameter $\mathbf{V}$ is called \emph{dual} variables, and the original parameters $\mathbf{\theta}$ are called \emph{primal} variables. Minimization of $\mc{J}(\theta)$ over primal variables $\theta$ has been transformed to a min-max problem on $\mc{L}(\theta,\mathbf{V})$ over both primal variables $\theta$ and dual variables $\mathbf{V}$. The min-max problem in this new form is called \emph{saddle-point problem}. Specifically, the saddle-point problem can be solved by seeking a point, which is not only the minimum point along primal domain ($\theta$) but also the maximum point along dual domain ($\mathbf{V}$). This point, denoted as $(\theta^*,\mathbf{V}^*)$, is called the \emph{saddle point} for the domain of $(\theta,\mathbf{V})$ [2]. Importantly, we find that the sum over $t$ has been taken outside of logarithm. As a result, we can now apply SGD or its variation in learning the model.

\subsection{Profile Surface of $\mc{L(\theta,\mathbf{V})}$}
We now draw the profile of cost function $\mc{L}(\theta,\mathbf{V})$ in Figure \ref{Fig:Saddle2D}. Like the profile surface of $\mc{J}(\theta)$ in Section IV-C, we need to first get $\theta^0$ by solving the supervised problem $\theta^0 = \arg\max_\theta \sum_{t=1}^{T} \ln{p_{t,\theta}(y_t|x_t)}$. Then, we inference $V^0$ by 
\begin{align}
	\mathbf{V^0}= -\frac{1}{\mathbf{\overline{P_{LM}}}^T (\theta^0)} 
    \label{Equ:MinimaUnigram}
	\end{align}
With the supervised optima $\theta^0, \mathbf{V^0}$, we can plot the 3-D cost function surface around this optima. We randomly choose two direction $(\theta^1-\theta^0)$ and $(\mathbf{V^1} - \mathbf{V^0})$, where $\theta^1$ and $\mathbf{V^1}$ are random vectors with same size of $\theta^0$ and $\mathbf{V^0}$. Then we plot on the 3-D space the point $(\lambda_p, \lambda_d, Z(\lambda_p, \lambda_d))$, with $Z(\lambda_p, \lambda_d) = \mc{J}(\theta^* + \lambda_p(\theta_1-\theta^0), \mathbf{V^0} + \lambda_d(\mathbf{V_1} - \mathbf{V^0}))$. The Figure \ref{Fig:Saddle2D} shows the plotting results. 

Comparing surface of original surface in Figure \ref{Fig:Surface2DUnigram} with surface of transformed cost function in Figure \ref{Fig:Saddle2D}, we can see that, after the primal-dual reformulation, the barrier in original cost function $\mc{J}(\theta)$ disappears. Thus the optimization become much easier on the new cost function.

\begin{figure*}[t!]
	\centering
	\begin{subfigure}[t]{0.3\textwidth}
		\includegraphics[width=\textwidth]{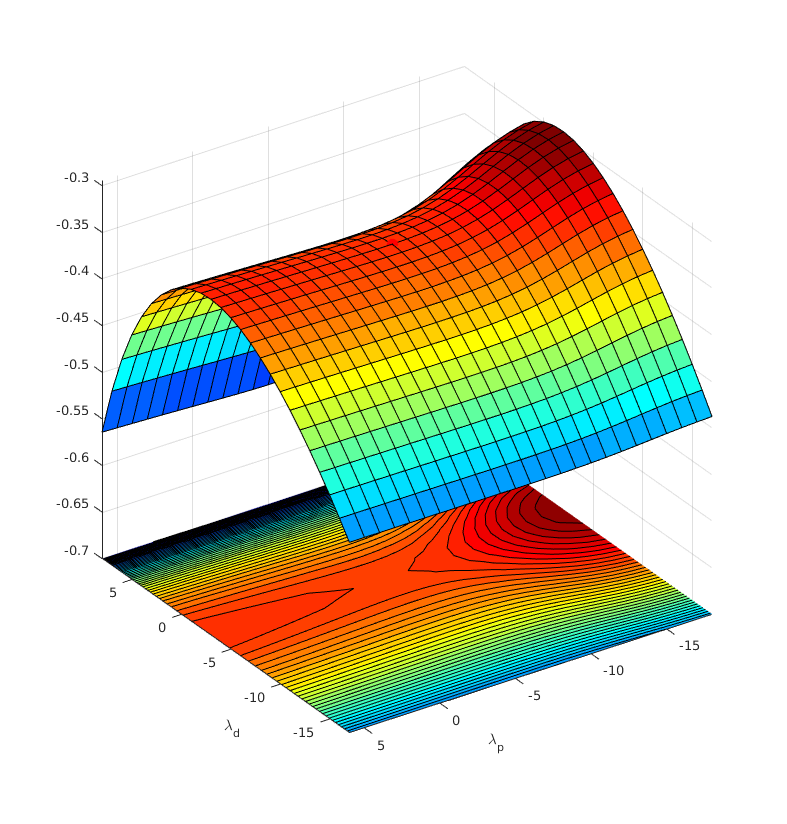}
		\caption{}
	\end{subfigure}	
	\quad
	\begin{subfigure}[t]{0.3\textwidth}
		\includegraphics[width=\textwidth]{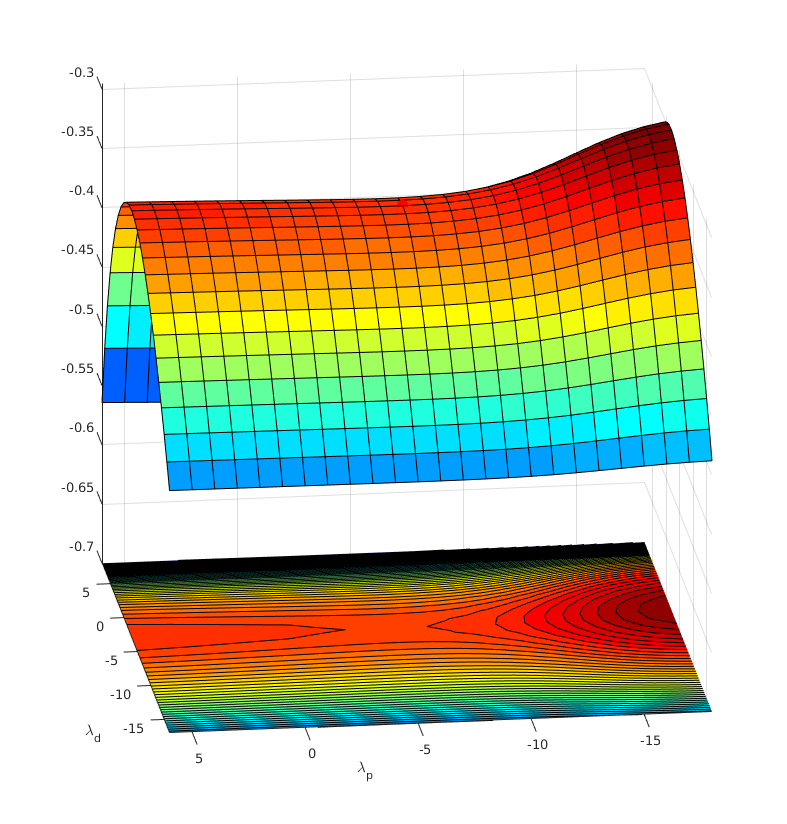}
		\caption{}
	\end{subfigure}
	\quad	
	\begin{subfigure}[t]{0.3\textwidth}
		\includegraphics[width=\textwidth]{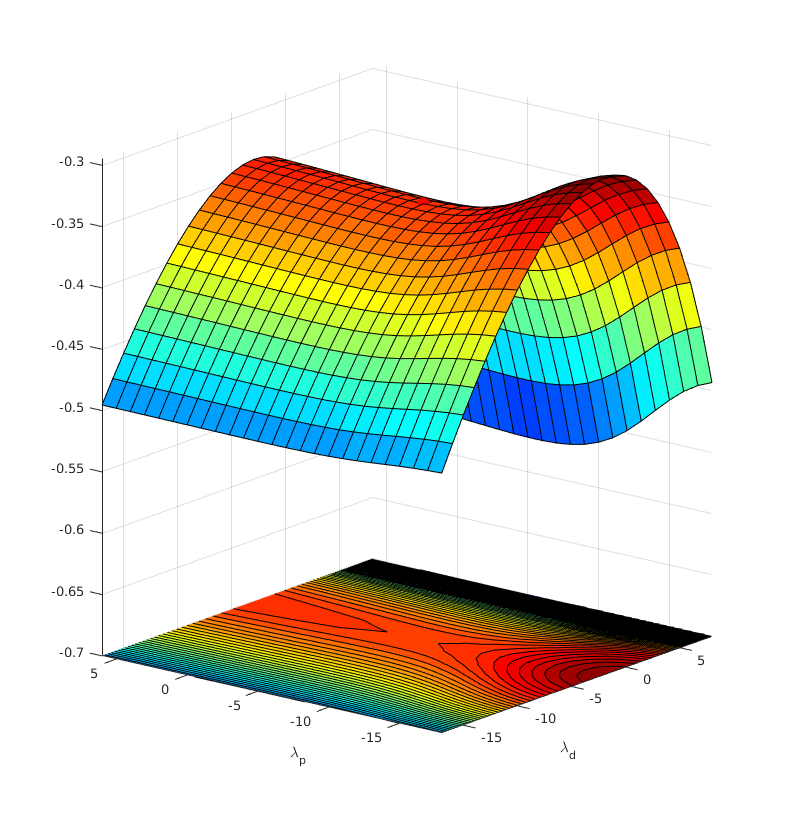}
		\caption{}
	\end{subfigure}	
	\quad
	\begin{subfigure}[t]{0.3\textwidth}
		\includegraphics[width=\textwidth]{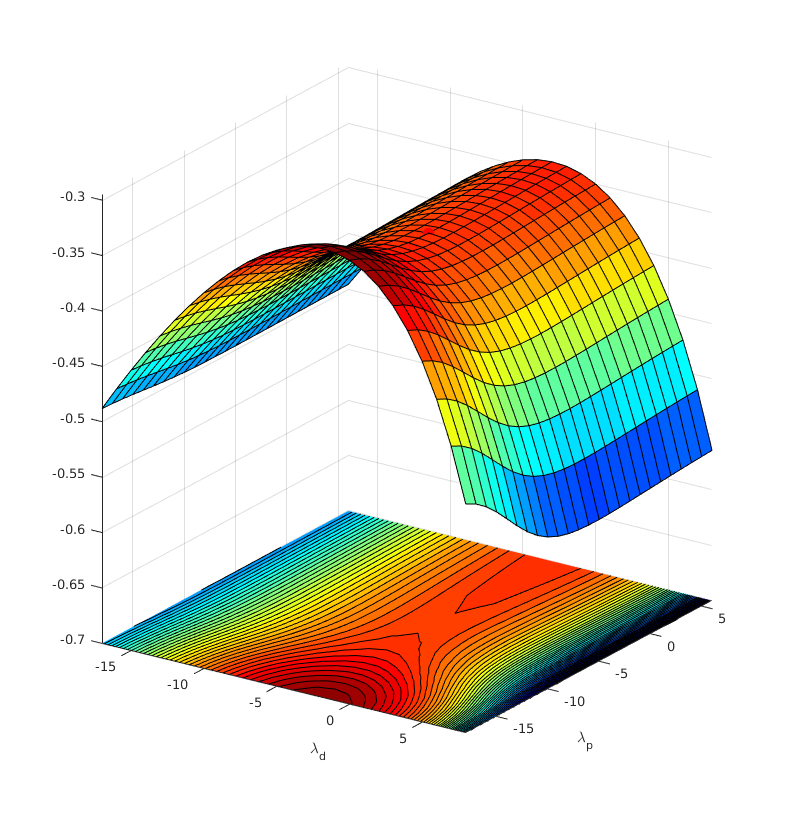}
		\caption{}
	\end{subfigure}	
	\quad
	\begin{subfigure}[t]{0.3\textwidth}
		\includegraphics[width=\textwidth]{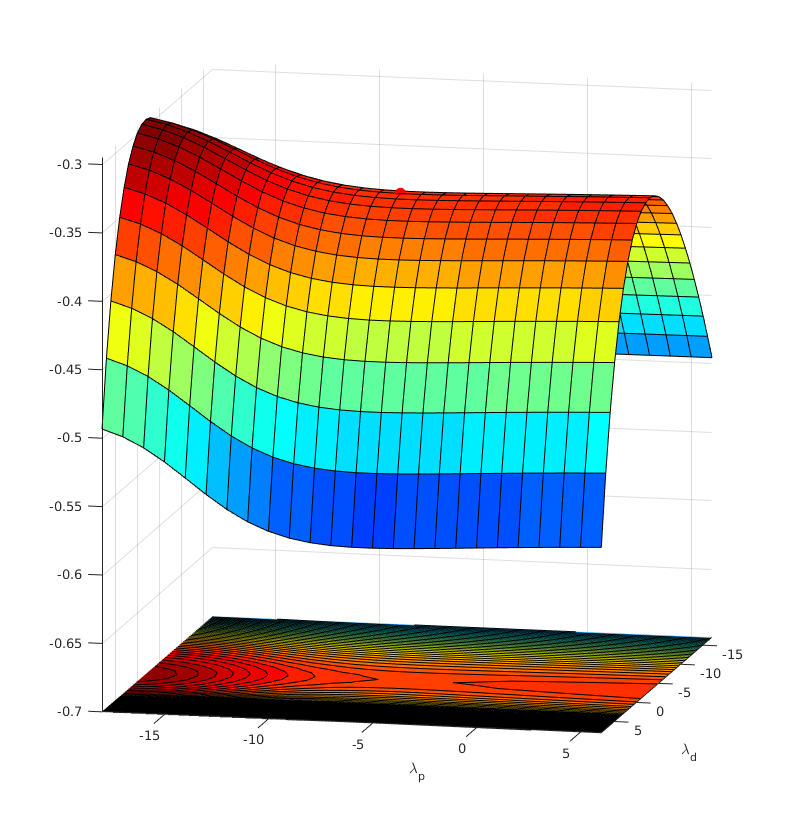}
		\caption{}
	\end{subfigure}
	\quad	
	\begin{subfigure}[t]{0.3\textwidth}
		\includegraphics[width=\textwidth]{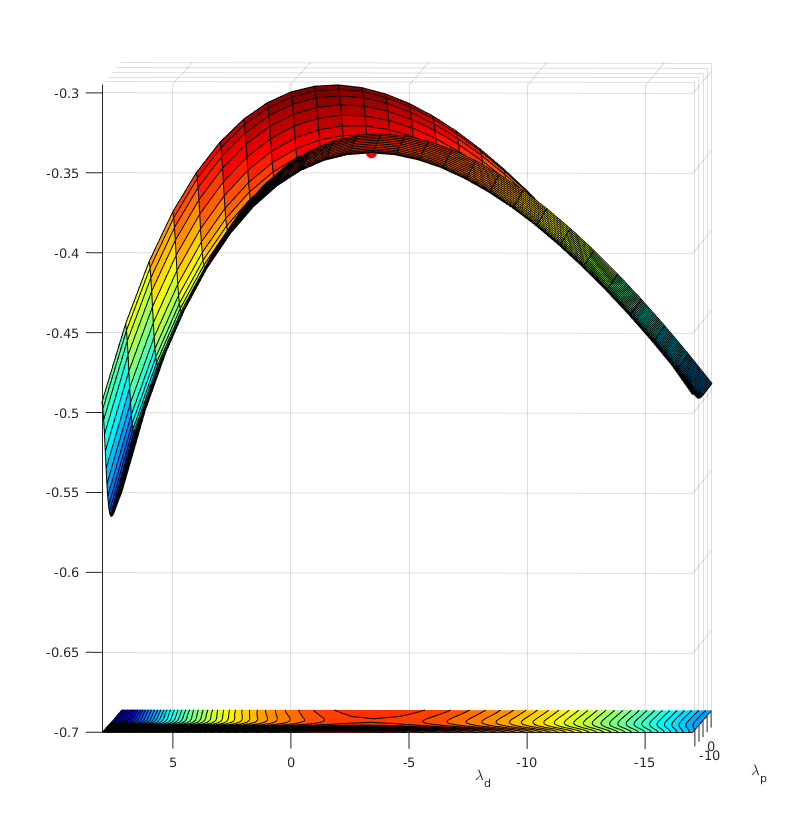}
		\caption{}
	\end{subfigure}	
	\caption{$\mc{L}(\theta, \mathbf{V})$ profile surface in the primal-dual domain of the Unigram case. The six sub-figures show the same profile from six different angles, spinning clock-wise from (a)-(f). The red dots indicate the saddle point (global optima). 
	}
	\label{Fig:Saddle2D}
\end{figure*}

\subsection{Solution with Primal-Dual Optimization}

In this section, we proceed to find the solution to the saddle-point problem in Eqn. (\ref{Equ:finalMatrix}).

To solve a min-max problem, our idea is to minimize $\mc{L}(\theta,\mathbf{V})$ with respect to $\theta$, while, at the same time, maximize $\mc{L}(\theta,\mathbf{V})$ with respect to $\mathbf{V}$. Specifically, we solve the min-max problem by stochastic gradient \emph{descent} of $\mc{L}(\theta,\mathbf{V})$ over $\theta$ while by stochastic \emph{accent} of $\mc{L}(\theta,\mathbf{V})$ over $\mathbf{V}$ with following update equations:
    \begin{align}
    	&\mathbf{\theta}_\tau 
		=
		\mathbf{\theta}_{\tau-1} -\mu_\theta \frac{\partial \mc{L}}{\partial \mathbf{\theta}} \bigg |_{\mathbf{\theta}=\mathbf{\theta}_{\tau-1}, \mathbf{V}=\mathbf{V}_{\tau-1} }
        \nn\\
        &\mathbf{V}_\tau 
		=
		\mathbf{V}_{\tau-1}  + \mu_V \frac{\partial \mc{L}}{\partial \mathbf{V}} \bigg |_{\mathbf{\theta}=\mathbf{\theta}_{\tau-1}, \mathbf{V}=\mathbf{V}_{\tau-1} }
        \label{Equ:update}
    \end{align}
The updates happens at time $\tau$. And $\mu_\theta, \mu_{V}>0$ are learning rate for primal variables and dual variables, respectively. Note that we update two of primal and dual variables in a synchronized manor. 

Now we derive the close form of $\frac{\partial \mc{L}}{\partial \mathbf{\theta}}$ and $\frac{\partial \mc{L}}{\partial \mathbf{V}}$. First, for $\frac{\partial \mc{L}}{\partial \mathbf{\theta}}$, from Eqn. (\ref{Equ:finalMatrix}), we take partial derivative: 
	\begin{align}
    \frac{\partial \mc{L}}{\partial \mathbf{\theta}}
    =
    \frac{\partial}{\partial \mathbf{\theta}}
    \bigg \{
    [\mathbf{P_{LM}} \odot  \mathbf{V}]  \mathbf{\overline{P_{LM}}} (\theta)  
    +\mathbf{P_{LM}} \ln{(-\mathbf{V})}^T  
    \bigg \}
    =
    [\mathbf{P_{LM}} \odot  \mathbf{V}] 
    \frac{\partial \mathbf{\overline{P_{LM}}} (\theta)}{\partial \mathbf{\theta}}
    \label{Equ:theta_dir}
    \end{align}
where the term $\frac{\partial \mathbf{\overline{P_{LM}}} (\theta)}{\partial \mathbf{\theta}}$ in above can be decomposed and calculated in element-wise. Recall that $\mathbf{P_{LM}}  = [p_0,p_1]$ is constant,
 $\mathbf{\overline{P_{LM}}} (\theta) =[\frac{1}{T}\sum_{t=1}^T p_{\theta,t}(0),~~\frac{1}{T} \sum_{t=1}^T p_{\theta,t}(1) ]^T$ and $\mathbf{\theta}=[w^a,w^b]^T$. From matrix derivative rules we know:
	\begin{align}
        \frac{\partial \mathbf{\overline{P_{LM}}} (\theta) }{\partial \mathbf{\theta}}
        =
        \frac{1}{T}\sum_{t=1}^T
        \begin{bmatrix}
        {\partial p_{t,\theta}(0)}/{\partial w^a} & {\partial p_{t,\theta}(0)}/{\partial w^b}  \\ 
        {\partial p_{t,\theta}(1)}/{\partial w^a} & {\partial p_{t,\theta}(1)}/{\partial w^b}
        \end{bmatrix}
    \label{Equ:Whole_dir}
    \end{align}
where $p_{t,\theta}(0)$ and $p_{t,\theta}(1)$ are defined in Eqn. (\ref{Equ:logLinearModel}) being log-linear classifier model for classes $0$ and $1$, respectively. The four derivatives in Eqn. (\ref{Equ:Whole_dir}) can be calculated from definition in Eqn. (\ref{Equ:logLinearModel}). For example: 
	\begin{align}
    \frac{\partial p_{t,\theta}(0)}{\partial w^a}
    &=
    \frac{\partial}{\partial w^a} \bigg \{
    \frac{ e^{\gamma w^a x^a_t} }{e^{\gamma w^a x^a_t} + e^{\gamma w^b x^b_t}} \bigg \}
    \nn\\
    &=
    \frac{\partial}{\partial w^a} \bigg \{
    \frac{ 1 }{1 + e^{ - (\gamma w^a x^a_t - \gamma w^b x^b_t) }} \bigg \}
    \nn\\
    &=
    \frac{\partial}{\partial w^a}
    \sigma(\gamma w^a x^a_t - \gamma w^b x^b_t) 
    \nn\\
    &=
    \sigma(\gamma w^a x^a_t - \gamma w^b x^b_t)(1-\sigma(\gamma w^a x^a_t - \gamma w^b x^b_t)) \frac{\partial (\gamma w^a x^a_t - \gamma w^b x^b_t)}{\partial w^a}
    \nn\\
    &=
    p_{t,\theta}(0) (1-p_{t,\theta}(0)) (\gamma x^a_t) 
    \nn\\
    &=
    \gamma p_{t,\theta}(0) p_{t,\theta}(1)  x^a_t
    \label{Equ:dir1}
    \end{align}
In above derivation, $\sigma(\cdot)$ indicate sigmoid function that $\sigma(a)=\frac{ 1 }{1 + e^{-a}}$. Note that we use the equation  $\frac{\partial \sigma(a)}{\partial a}= \sigma(a)(1-\sigma(a))$ in above derivation. Similarly, we can get all the derivatives in matrix of \ref{Equ:Whole_dir} as:
\begin{align}
	&\frac{\partial p_{t,\theta}(0)}{\partial w^a}
    =
    \gamma p_{t,\theta}(0) p_{t,\theta}(1)  (x^a_t)
    \nn\\
    &\frac{\partial p_{t,\theta}(0)}{\partial w^b}
    =
    \gamma p_{t,\theta}(0) p_{t,\theta}(1)  (-x^b_t)
    \nn\\
    &\frac{\partial p_{t,\theta}(1)}{\partial w^a}
    =
    \gamma p_{t,\theta}(0) p_{t,\theta}(1)  (-x^a_t)
    \nn\\
    &\frac{\partial p_{t,\theta}(1)}{\partial w^b}
    =
    \gamma p_{t,\theta}(0) p_{t,\theta}(1)  (x^b_t)
    \label{Equ:Element_dir}
    \end{align}
Then we substitute all derivative terms in Eqn. (\ref{Equ:Whole_dir}) using (\ref{Equ:Element_dir}):
    \begin{align}
    \frac{\partial \mathbf{\overline{P_{LM}}} (\theta) }{\partial \mathbf{\theta}}
        &= 
        \frac{1}{T}\sum_{t=1}^T \gamma p_{t,\theta}(0) p_{t,\theta}(1) \mathbf{X}
        \nn\\
        &where~~~
        \mathbf{X}
        =
        \begin{bmatrix}
        ~x_t^a & -x_t^b  \\ 
        -x_t^a & ~x_t^b
        \end{bmatrix}
    \label{Equ:PDiv}
	\end{align}
Thus we get the final close form of $\frac{\partial \mc{L}}{\partial \mathbf{\theta}}$:
	\begin{align}
        \frac{\partial \mc{L}}{\partial \mathbf{\theta}}
        =
		\frac{1}{T} \sum_{t=1}^{T} \bigg \{
\gamma p_{t-1,\theta}(0) p_{t-1,\theta}(1)   [\mathbf{P_{LM}} \odot \mathbf{V}] \mathbf{X} \bigg \}
        \label{Equ:finalDiv_theta}
     \end{align}
Then for $\frac{\partial \mc{L}}{\partial \mathbf{V}}$, recall that $\mathbf{V} = [v_0, v_1]$. Simply use matrix derivative rules and we can get:
     \begin{align}
        \frac{\partial \mc{L}}{\partial \mathbf{V}}
        &=
        \frac{\partial}{\partial \mathbf{V}}
    \bigg \{
    [\mathbf{P_{LM}} \odot  \mathbf{V}]  \mathbf{\overline{P_{LM}}} (\theta)  
    +\mathbf{P_{LM}} \ln{(-\mathbf{V})}^T  
    \bigg \}
    \nn\\
    	&=
        \mathbf{P_{LM}} \odot \mathbf{\overline{P_{LM}}}^T (\theta) + \frac{\mathbf{P_{LM}}}{\mathbf{V}}
        \label{Equ:finalDiv_l}
	\end{align}

For mini-batch updates, we simply replace the sum over total data $T$ with sum over the mini-batch data $B=(\mathbf{x_s},…,\mathbf{x_{s+N-1}})$. For each update, a sequence of data with size $N$ is randomly taken from $\mc{X}$ as a mini-batch. Then we use Formula (\ref{Equ:update}) to update both primal and dual variable with close form (\ref{Equ:finalDiv_theta}) and (\ref{Equ:finalDiv_l}). Repeating this process until convergence. We describe the this process with mini-batch gradient in Algorithm \ref{Alg:SPDG}.

\subsection{The Weakness of Unigram}

However, we found the language model of Unigram is too weak. It does not capture the sequence structure which is important information for unsupervised classification. As in \cite{Chen2016}, the sequence prior of output labels is a important structure to employ in unsupervised problems, especially when the order is high. However, the Unigram language model does not capture such structure.
 
However, the experiment results turns out to be unsatisfied. The average classification error rate is $30.8\%$, which is just about the error rate of majority guess. The model constantly predicts class with label 1, regardless of the input data. In experiment we conduct the experiments on the binary classification task with the dataset we created for $(\mc{X}, \mc{Y})$.    

 Next, we will walk through this problem by using Bigram language model, the simplest language model with sequence prior statistics. The main procedure of the derivation is the same with Unigram we just did in previous sections. So it can be easier for readers to apply to Bigram case, in spite of more mathematics will be involved.

\begin{algorithm}[tb]
	\renewcommand{\algorithmicrequire}{\textbf{Inputs:}}
	\renewcommand{\algorithmicensure}{\textbf{Outputs:}}
	\caption{Stochastic Primal-Dual Gradient Method}
	\label{Alg:SPDG}
	\begin{algorithmic}[1]	
		\STATE {\bf Input data:} $ \mc{X}=(\mathbf{x_1},\dots, \mathbf{x_T})$ and $\mathbf{P_{\LM}}$.
		\STATE Initialize $\theta$ and $\mathbf{V}$ from random numbers where the elements of $\mathbf{V}$ are negative
		\REPEAT
			\STATE Randomly sample a mini-batch of $B$ subsequences of length $N$, i.e.,
			$B=(\mathbf{x_s},\dots,\mathbf{x_{s+N-1}})$.
			\STATE Compute the stochastic gradients ${\partial \mc{L}}/{\partial \mathbf{\theta}}$ and ${\partial \mc{L}/}{\partial \mathbf{V}}$ for the subsequence in the mini-batch using Formula (\ref{Equ:finalDiv_theta}) and (\ref{Equ:finalDiv_l}).
			\STATE Update $\theta$ and $\mathbf{V}$ according to
			{\small
				\begin{align}
					\theta &\leftarrow \theta - \mu_{\theta} \frac{\partial \mc{L}}{\partial \mathbf{\theta}}, \qquad
					\mathbf{V} \leftarrow \mathbf{V} + \mu_V \frac{\partial \mc{L}}{\partial \mathbf{V}} \nn
				\end{align}			
			}
		\UNTIL{convergence or a certain stopping condition is met}
	\end{algorithmic}
\end{algorithm}

\section{Extension to the Case with Sequence Prior Statistics}
\label{SECTION:Extension}
\subsection{Problem Formulation for the Sequential Bigram Case}

In Bigram case, the problem formulation and approach is the same as Unigram case that we also use model for classifier as in formula \ref{Equ:logLinearModel}.  except using Bigram language models. Here below are the three main steps accordingly.

\textbf{Step 1: Prior language model statistics (Bigram)} We will create synthetic data with arbitrary given language model. First, we create output labels $\mc{Y} = (y_1,y_2,\dots,y_T)$ from a 1st-order Markov model that is based on a given transition probability:
	\begin{align}
		&p(y_1,y_2,\dots,y_T)
        =
					\prod_{t=1}^T p(y_t | y_{t-1})
		\label{Equ:1OrderMarkov}
	\end{align}
$p(y_t|y_{t-1})$ denotes conditional probability of the transition between adjacent labels $y_{t-1}$ to $y_t$, which is given in transition matrix $\mathbf{P_{Trans}}$. Then, the language model $\mathbf{P_{LM}}$ can be calculated from the transition matrix:
	\begin{align}
		&\mathbf{P_{LM}}
        \defeq
        \begin{bmatrix}
        p_{00} & p_{01} \\ 
        p_{10} & p_{11}
        \end{bmatrix}
        =diag(\mathbf{p_{st}}) \mathbf{P_{Trans}}
        \label{Equ:Bigram_PLM}
	\end{align}
$\mathbf{p_{st}}$ denotes the steady state vector of transition $\mathbf{P_{Trans}}$. The creation of data, as well as connection between language model and transition matrix will be explained in detail in experiment section VII.

\textbf{Step 2: Classifier output statistics (Bigram)} Due to first order dependency, the probability of data being assigned with each label by the classifier in Eqn. (\ref{Equ:logLinearModel}) can be calculated by:
\begin{align}
		\mathbf{\overline{P_{LM}}} (\theta) 
		&=
        \begin{bmatrix}
        \frac{1}{T}\sum_{t=1}^T p_{\theta,t}(0)p_{\theta,t-1}(0), ~~\frac{1}{T}\sum_{t=1}^T p_{\theta,t}(0)p_{\theta,t-1}(1)  \\ 
        \frac{1}{T}\sum_{t=1}^T p_{\theta,t}(1)p_{\theta,t-1}(0),~~ \frac{1}{T}\sum_{t=1}^T p_{\theta,t}(1)p_{\theta,t-1}(1)
        \end{bmatrix}
        \nn\\
        &=
        \frac{1}{T}\sum_{t=1}^T
        \begin{bmatrix}
        p_{\theta,t-1}(0) \\
        p_{\theta,t-1}(1)
        \end{bmatrix}
        \begin{bmatrix}
        p_{\theta,t}(0) \\
        p_{\theta,t}(1)
        \end{bmatrix}^T
        \nn\\
        &=
        \frac{1}{T}\sum_{t=1}^T \mathbf{P_{t-1,\theta}} \mathbf{P_{t,\theta}}^T
        \nn\\
        where ~~~~
        &\mathbf{P_{t,\theta}}
        =
        \begin{bmatrix}
        p_{t,\theta}(0) ~~
        p_{t,\theta}(1)
        \end{bmatrix}^T
        \label{Equ:classifierstat}
	\end{align}

\textbf{Step 3: Matching the two statistics} As with Unigram, we still use KL-divergence to measure the difference, and obtain cost function :
	\begin{align}
	\mc{J}(\theta)
	=
    -\left \langle \mathbf{P_{LM}} ~, ~\ln{\mathbf{\overline{P_{LM}}} (\theta)} \right \rangle
    =
	  -\sum_{i,j \in \{ 0,1\} } p_{ij} 	\ln{ \frac{1}{T}\sum_{t=1}^T [p_{t-1,\theta}(i)  p_{t,\theta}(j)]} 
    \label{Equ:Cost}
	\end{align}
where the notation $\left \langle \mathbf{A},\mathbf{B} \right \rangle=\sum_i \sum_j a_{ij}b_{ij}$ stands for the inner product of two matrices $\mathbf{A}$ and $\mathbf{B}$, specifically, sum of the products of the corresponding components of two matrices. Thus, we formulate our problem as an optimization problem finding the best classifier parameters $\theta^*$:
\begin{align}
	\theta^* = \arg \min_\theta \mc{J}(\theta)
    \label{Equ:minCostBigram}
	\end{align}

We again observe the sum inside logarithm. With using Lemma \ref{Equ:lemmaConj}, we can transform this problem to its equivalent saddle-point problem:
	\begin{align*}
	\min_{\theta}{\mc{J}(\theta)} 
    &=
    \min_{\theta}{\left \{ \sum_{i,j \in \{ 0,1\} }  p_{ij} 	\max_{v_{ij}<0}{\left [  v_{ij}\frac{1}{T}\sum_{t=1}^T [p_{t-1,\theta}(i)  p_{t,\theta}(j)]+\ln{(-v_{ij})} \right ]}  \right \}}
    \nn\\
    &=
    \min_{\theta}{ \max_{v_{ij}<0}{  \left \{ \sum_{i,j \in \{ 0,1\} }  p_{ij} v_{ij}\frac{1}{T}\sum_{t=1}^T [p_{t-1,\theta}(i)  p_{t,\theta}(j)]+ p_{ij} \ln{(-v_{ij})}  \right \} } }
    \nn\\
    &=
    \min_{\theta} \max_{v_{ij}<0}  \Bigg\{ \frac{1}{T}  \sum_{t=1}^{T}  \sum_{i,j \in \{ 0,1\} } p_{ij} v_{ij} p_{t-1,\theta}(i) p_{t,\theta}(j)
    + \sum_{i,j \in \{ 0,1\}} p_{ij} \ln{(-v_{ij})}  \Bigg\} 
    \nn\\
    &=
    \min_{\mathbf{\theta}} \max_{\mathbf{V}<0}  \bigg \{ 
    \frac{1}{T} \sum_{t=1}^{T} (\mathbf{p_{t-1,\theta}^T}   [\mathbf{P_{LM}} \odot  \mathbf{V}]  \mathbf{p_{t,\theta}} )   
    +\left \langle \mathbf{P_{LM}}, \ln{(-\mathbf{V})} \right \rangle
    \bigg \}
    \nn\\
   \end{align*}
As in the case of Unigram, we define the new cost:
  \begin{align}
    \mc{L}(\theta,\mathbf{V}) 
    &= \frac{1}{T} \sum_{t=1}^{T} (\mathbf{p_{t-1,\theta}^T}   [\mathbf{P_{LM}} \odot  \mathbf{V}]  \mathbf{p_{t,\theta}} )   
    +\left \langle \mathbf{P_{LM}}, \ln{(-\mathbf{V})} \right \rangle
    \nn\\
    \mathbf{V}
    &=
    \begin{bmatrix}
    v_{00} ~~ v_{01}  \\ 
    v_{10} ~~ v_{11}
    \end{bmatrix}
    ~~~~
    \mathbf{P_{t,\theta}}
    =
    \begin{bmatrix}
    p_{t,\theta}(0)  \\ 
    p_{t,\theta}(1)
    \end{bmatrix}
    ~~~~
    \mathbf{P_{LM}}
    =
    \begin{bmatrix}
    p_{00} ~~ p_{01}  \\ 
    p_{10} ~~ p_{11}
    \end{bmatrix}
    ~~~~
    \mathbf{\theta}
    =
    \begin{bmatrix}
    w^a  \\ 
    w^b
    \end{bmatrix}
    \label{Equ:finalMat}
	\end{align}
$\theta$ is called \emph{primal} variables, and $\mathbf{V}$ is called \emph{dual} variables. Thus, the minimization problem of $\mc{J}(\theta)$ over primal variables $\theta$ is transformed to an equivalent saddle-point problem of $\mc{L}(\theta,\mathbf{V})$, and the sum over $t$ has been taken outside of logarithm. In other words, the optimal solution $(\theta^*,\mathbf{V}^*)$ to $\mc{L}(\theta,\mathbf{V})$ is called the saddle point [2], which is relative minimum along primal domain $\theta$ while is a relative maximum along dual domain $\mathbf{V}$. 

\subsection{Profile Surface of $\mc{L}(\theta,\mathbf{V})$}

As in the Unigram case, we also visualize the profile surface of cost function $\mc{L}(\theta,\mathbf{V})$ in Figure \ref{Fig:Saddle2DBigram}. Again, we need to first get $\theta^0$ by solving the supervised problem $\theta^0 = \arg\max_\theta \sum_{t=1}^{T} \ln{p_{t,\theta}(y_t|x_t)}$, and inference $V^0$ using 
\begin{align}
	\mathbf{V^0}= -\frac{T}{\sum_{t=1}^T \mathbf{P_{t-1,\theta^0}} \mathbf{P_{t,\theta^0}^T}}
    \label{Equ:MinimaBigram}
	\end{align}
Using the same method, we first randomly choose two direction $(\theta^1-\theta^0)$ and $(\mathbf{V^1} - \mathbf{V^0})$, where $\theta^1$ and $\mathbf{V^1}$ are random vectors with same size of $\theta^0$ and $\mathbf{V^0}$. Then we plot on the 3-D space the point $(\lambda_p, \lambda_d, Z(\lambda_p, \lambda_d))$, with $Z(\lambda_p, \lambda_d) = \mc{J}(\theta^* + \lambda_p(\theta_1-\theta^0), \mathbf{V^0} + \lambda_d(\mathbf{V_1} - \mathbf{V^0}))$. The Figure \ref{Fig:Saddle2DBigram} shows the plotting results.

\begin{figure*}[t!]
	\centering
	\begin{subfigure}[t]{0.28\textwidth}
		\includegraphics[width=\textwidth]{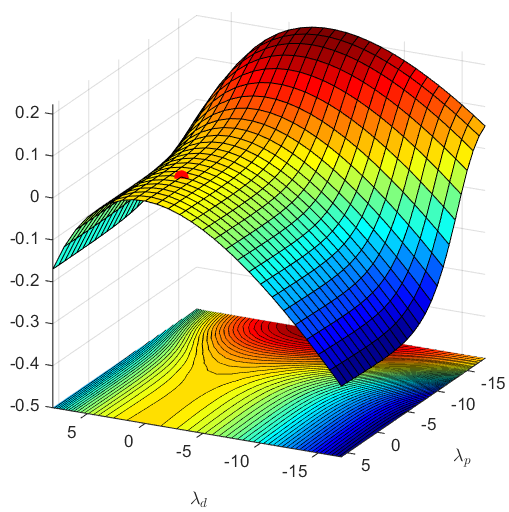}
		\caption{}
	\end{subfigure}	
	\quad
	\begin{subfigure}[t]{0.28\textwidth}
		\includegraphics[width=\textwidth]{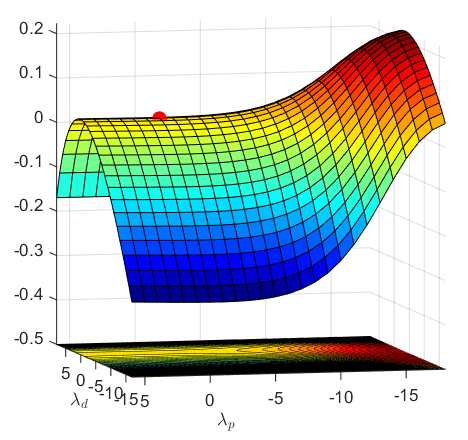}
		\caption{}
	\end{subfigure}
	\quad	
	\begin{subfigure}[t]{0.28\textwidth}
		\includegraphics[width=\textwidth]{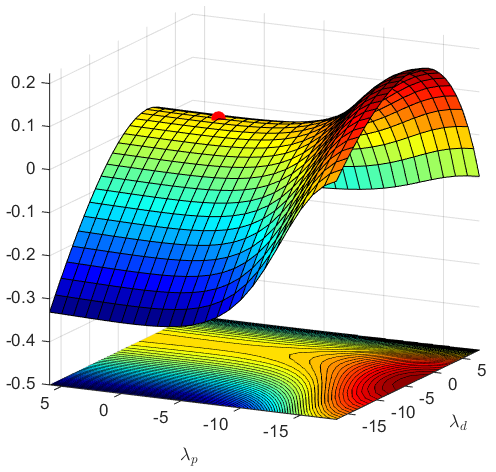}
		\caption{}
	\end{subfigure}	
	\quad
	\begin{subfigure}[t]{0.28\textwidth}
		\includegraphics[width=\textwidth]{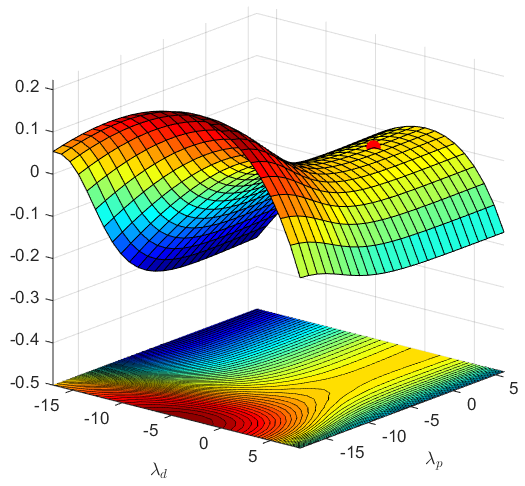}
		\caption{}
	\end{subfigure}	
	\quad
	\begin{subfigure}[t]{0.28\textwidth}
		\includegraphics[width=\textwidth]{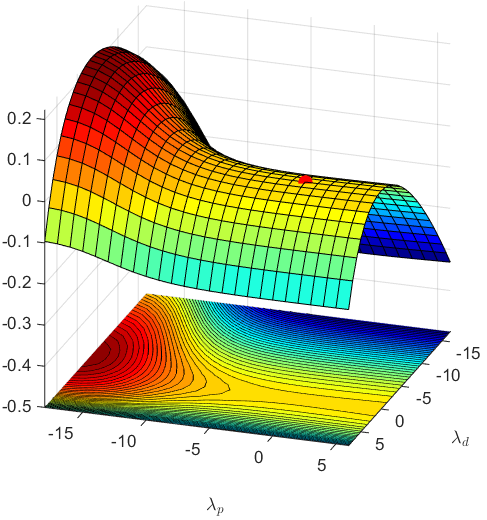}
		\caption{}
	\end{subfigure}
	\quad	
	\begin{subfigure}[t]{0.28\textwidth}
		\includegraphics[width=\textwidth]{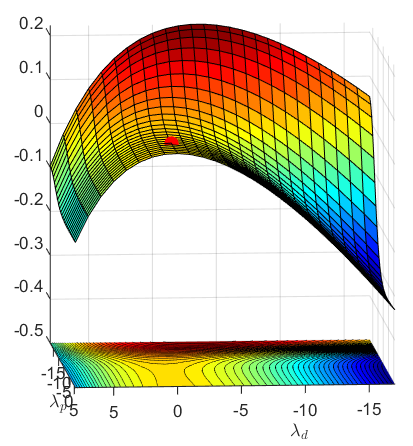}
		\caption{}
	\end{subfigure}	
	\caption{$\mc{L}(\theta, \mathbf{V})$ profile in the primal-dual domain of the Bigram case. The six sub-figures show the same profile from six different angles, spinning clock-wise from (a)-(f). The red dots indicate the saddle point (global optima). 
	}
	\label{Fig:Saddle2DBigram}
\end{figure*}

\subsection{Primal-Dual Optimization for Bigram Case}

Similarly, we solve the min-max problem by stochastic gradient \emph{descent} of $\mc{L}(\theta,V)$ over $\theta$ while by stochastic \emph{accent} of $\mc{L}(\theta,V)$ over $V$ with the same update equations:
    \begin{align}
    	&\mathbf{\theta}_\tau 
		=
		\mathbf{\theta}_{\tau-1} -\mu_\theta \frac{\partial \mc{L}}{\partial \mathbf{\theta}} \bigg |_{\mathbf{\theta}=\mathbf{\theta}_{\tau-1}, \mathbf{V}=\mathbf{V}_{\tau-1} }
        \nn\\
        &\mathbf{V}_\tau 
		=
		\mathbf{V}_{\tau-1}  + \mu_V \frac{\partial \mc{L}}{\partial \mathbf{V}} \bigg |_{\mathbf{\theta}=\mathbf{\theta}_{\tau-1}, \mathbf{V}=\mathbf{V}_{\tau-1} }
        \label{Equ:updateBigram}
    \end{align}
The updates happens at time $\tau$. And $\mu_\theta, \mu_{V}>0$ are learning rate for primal variables and dual variables, respectively. Note that we update two of primal and dual variables in a synchronized manor. 

By using Formula (\ref{Equ:PDiv}), we calculate ${\partial \mc{L}}/{\partial \mathbf{\theta}}$ and ${\partial \mc{L}/}{\partial \mathbf{V}}$ similar as the Unigram Case:
		\begin{align}
        &\frac{\partial \mc{L}}{\partial \mathbf{\theta}}
        =
		\frac{1}{T} \sum_{t=1}^{T} \bigg \{
		\gamma P_{t-1,\theta}(0) P_{t-1,\theta}(1) 
		(
	[ \mathbf{X_{t-1}^T}  (\mathbf{P_{LM}} \odot \mathbf{V}) \mathbf{P_{t,\theta}}
        + (\mathbf{P_{LM}} \odot \mathbf{V})^T    \mathbf{X_{t}^T} \mathbf{P_{t-1,\theta}} ]
        \bigg \}
        \nn\\
        &\frac{\partial \mc{L}}{\partial \mathbf{V}}
        =
        \frac{1}{T} \sum_{t=1}^{T} \bigg \{
        \mathbf{P_{LM}} \cdot (\mathbf{P_{t-1,\theta} P_{t,\theta}^T}) + \frac{\mathbf{P_{LM}}}{\mathbf{V}}
        \bigg \}
        \nn\\
        &~~~~~~~~~~~~~~~~~~~~~~~~~~~~~~~~~~~~~~~~~~~~~~~~~where~~
        \mathbf{X}
        =
        \begin{bmatrix}
        x_t^a & -x_t^b  \\ 
        -x_t^a & x_t^b
        \end{bmatrix}
        \label{Equ:finalUpdateBigram}
	\end{align}

As Unigram, we also use mini-batch to calculate each update during training in the Bigram case. It is the same as the Unigram case with mini-batch gradient in Algorithm \ref{Alg:SPDG}.

\section{Experiments}
\label{SECTION:Experiment}

\subsection{The Case of Unigram}

\begin{figure}[b!]
		\centering
   		\includegraphics[width = .8\linewidth]{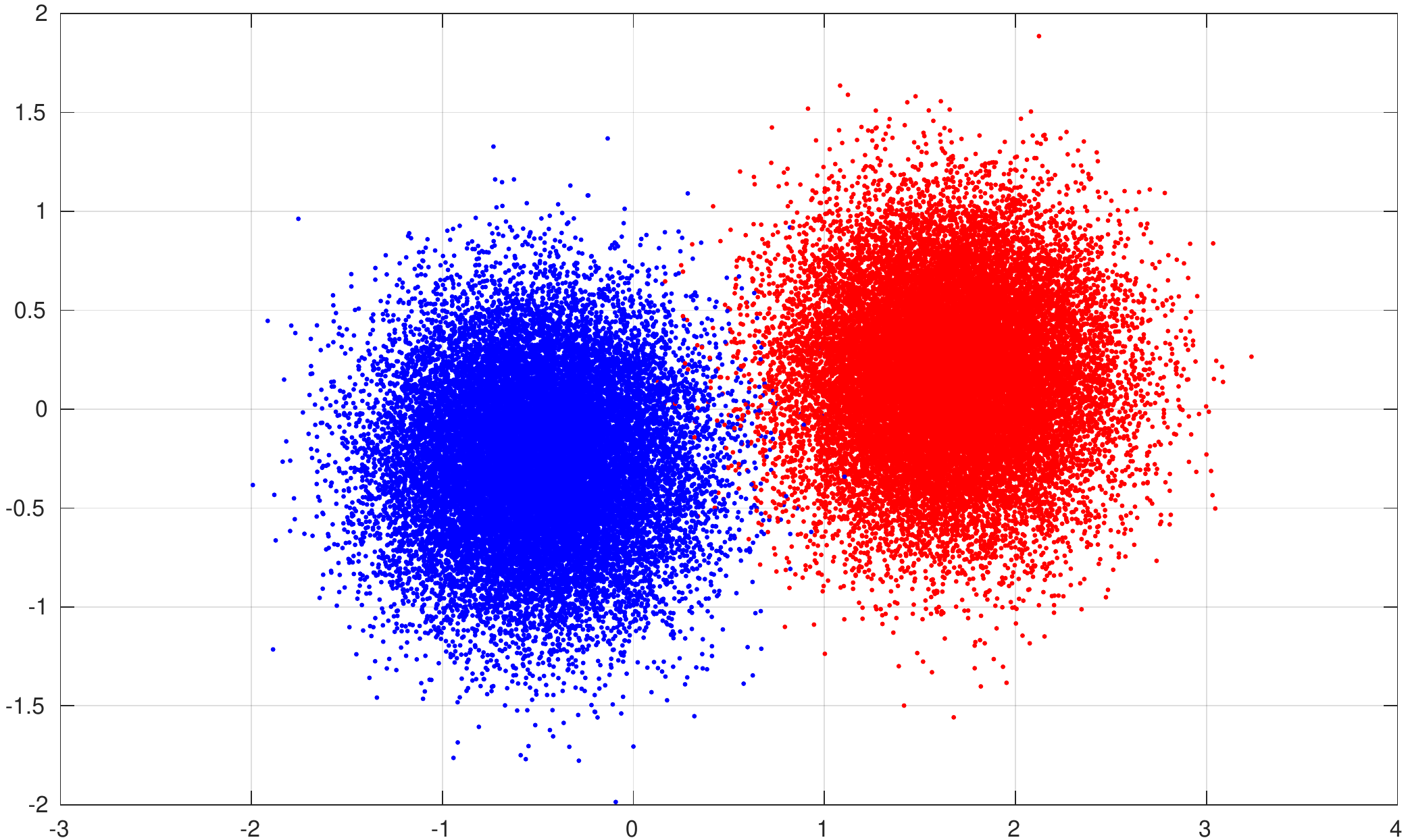}
   		\caption{Plot of data points of two classes.}
   		\label{fig:inputVisualization}
	\end{figure}

\textbf{Data Generation} We first create synthetic data with Unigram language model, where the data points and associated labels subject to a given Unigram language model, such as:
	\begin{align}
		&\mathbf{P_{LM}}
        =
        \begin{bmatrix}
        p_{0}   ~~
        p_{1}
        \end{bmatrix}
        =
        \begin{bmatrix}
        0.692   ~~
        0.308
        \end{bmatrix}
		\label{Equ:PLM_uni}
	\end{align}
That is 69.2\% of labels ``0'' and 30.8\% of label ``1'' in $\mc{Y} = (y_1,\dots,y_T)$.

Given the corresponding label of $y_t$, we can create the input points on a 2-dimensional grid $\mc{X} = (\mathbf{x_1} \dots \mathbf{x_T})$ by sampling over independent following Gaussian distributions:
	\begin{align*}
		&p(\mathbf{x_t} | y_t=0) \sim N(\mathbf{\mu_0} ; \mathbf{\Sigma})
        \nn\\
        &p(\mathbf{x_t} | y_t=1) \sim N(\mathbf{\mu_1} ; \mathbf{\Sigma})
	\end{align*}
where the parameters $\mathbf{\mu_0,\mu_1} \in \mathbb{R}^{2 \times 1}$ are centers of two classes data points, and $\mathbf{\Sigma}$ denotes the covariance matrix. In experiment, the central points $\mathbf{\mu_0}$ and $\mathbf{\mu_1}$ are randomly generated, for example, $\mathbf{\mu_0}=[-0.504,-0.264]^T$ and $\mathbf{\mu_1}=[1.646,0.181]^T$. Particularly, $\mathbf{\Sigma}$ is manually chosen to make sure the data points of two classes slightly ``blend'' into each other, because we need to insure that an unique optimal solution can be reached. So we set $\mathbf{\Sigma}=0.4 \cdot \mathbf{I}$. Following above settings, we sample total size of $T=60,000$ data points, which is then split to 50,000 for training, 5,000 for testing and 5,000 for validation. Figure \ref{fig:inputVisualization} shows examples of the plot of data points $\mc{X}$, with blue points being class ``0'' and red ones being class ``1''.

\textbf{Experimental Results} To validate the analysis and to confirm mathematical derivation, we implement Algorithm \ref{Alg:SPDG} with all hyper-parameters list in Table \ref{Tab:hyperparameters}. We conduct the experiments on the binary classification task with the dataset we created for $(\mc{X}, \mc{Y})$. However, the results turns out to be unsatisfactory. The average classification error rate is $30.8\%$, which is just about the error rate of majority guess: it constantly predicts class ``0'' regardless of the input data. We repeated the above experiment 5 times and collect the results as shown in Table \ref{Tab:UnigramResults}, which further validate that this method is nothing more than guessing majority class. 

\textbf{Problems of Non-sequence Language Model} The reason for bad results is that the Unigram language model is too weak. Unigram is a non-sequence language model that loses important information for unsupervised classification problem. As in \cite{Chen2016}, the sequence prior of output labels is a important structure to employ in unsupervised problems, especially when the order is high. However, the Unigram language model does not capture such structure. In next experiment we use Bigram language model that will show much better results.

\begin{table}[t]

\caption{The hyper-parameters used in training}
\begin{center}
\begin{tabular}{llr}
\hline
Hyper-parameter & Notation & Value  \\
\hline
Learning Rate 			&$\mu_\theta$         	& $10^{-6}$ \\
Learning Rate 			&$\mu_V$         		& $10^{-4}$ \\
Mini-batch Size 		&$N$       			& 10   \\
Training Size 			&$T_{trn}$	  		& 50,000   \\
Test Size 				&$T_{tst}$	  			& 5,000   \\
Validation Size 		&$T_{val}$	  		& 5,000  \\
\hline
\end{tabular}
\end{center}
\label{Tab:hyperparameters}
\end{table}

\begin{table}[t]

\caption{Experiment Results using Unigram}
\begin{center}
\begin{tabular}{llr}
\hline
 & $\mathbf{P_{LM}}$ & Error Rate  \\
\hline
Exp1 	&[0.692, 0.307]     &  30.7\%\\
Exp2  	&[0.385, 0.615]     &  38.5\%\\
Exp3  	&[0.583, 0.417]     &  41.6\%\\
Exp4  	&[0.692, 0.308]     &  30.8\%\\
Exp5  	&[0.667, 0.333]     &  33.3\%\\
\hline
\end{tabular}
\end{center}
\label{Tab:UnigramResults}
\end{table}

\subsection{The Case of Bigram}

\begin{figure}[t!]
		\centering
   		\includegraphics[width = .6\linewidth]{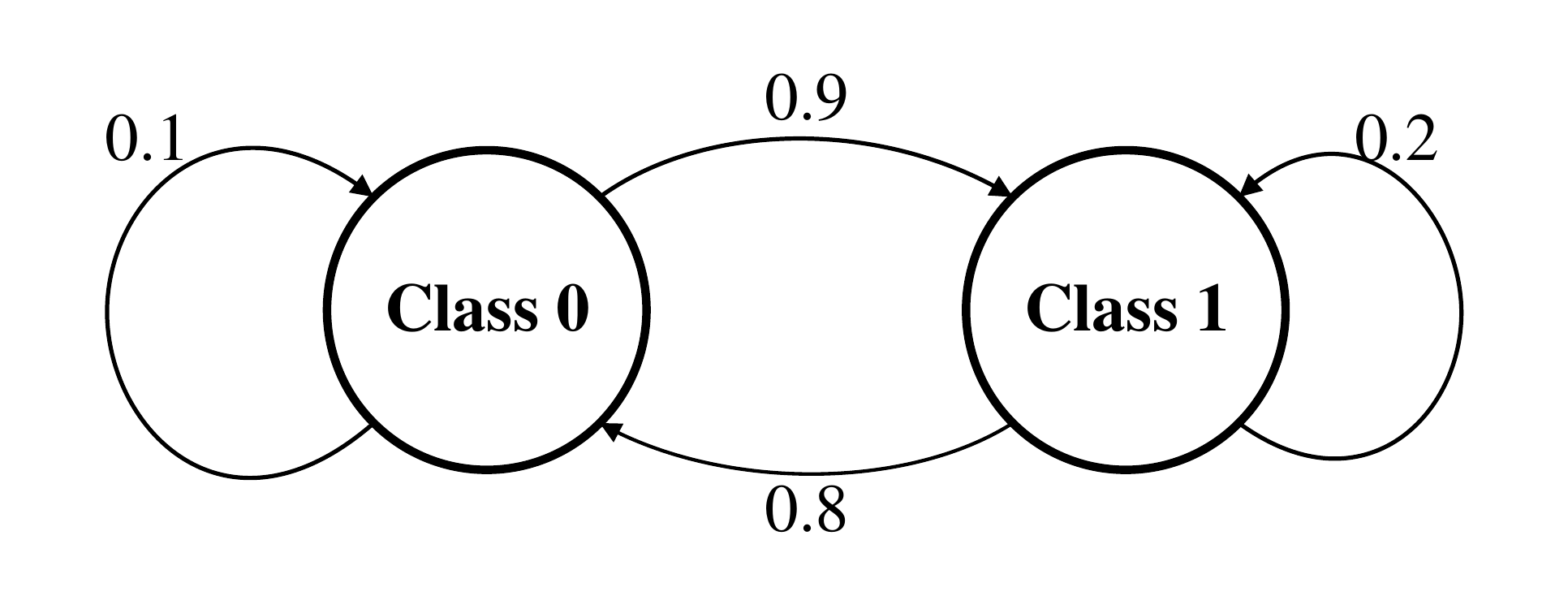}
  		 \vspace{-2mm}
   		\caption{The transition probability of output observation.}
   		\label{fig:transition}
	\end{figure}

\textbf{Data Generation} Different from Unigram data, now we need to create synthetic data that has additional Markov property: that the class label of future data point is solely dependent on the current label. We decided to use Hidden Markov  Model to create the dataset, by sampling from a first-order HMM. First, we sample $\mc{Y} = (y_1,\dots,y_T)$ from a given transition matrix $\mathbf{P_{Trains}}$. For example, 
	\begin{align}
		&\mathbf{P_{Trains}}
        =
        \begin{bmatrix}
        p(y_{t}=0|y_{t-1}=0) & p(y_{t}=1|y_{t-1}=0) \\ 
        p(y_{t}=0|y_{t-1}=1) & p(y_{t}=1|y_{t-1}=1)
        \end{bmatrix}
        =
        \begin{bmatrix}
        0.6 & 0.4 \\ 
        0.9 & 0.1
        \end{bmatrix}.
        \label{Equ:TransExample}
	\end{align}
where $P_{Trans}(i,j)=p(y_t=j|y_{t-1}=i)$. This transition relations are visualized in graph of Figure \ref{fig:transition}.

Then, similar to Unigram case, we use two Gaussian emission models to sample the data for $\mc{X} = (\mathbf{x_1} \dots \mathbf{x_T})$ given the sampled labels $\mc{Y} = (y_1,\dots,y_T)$: 
	\begin{align*}
    	&p(y_t=i|y_{t-1}=j)=P_{Trans}(i,j)
        \nn\\
        &~~~~~~~~~~~~~~~~~~and
        \nn\\
		&p(\mathbf{x_t} | y_t=0) \sim N(\mathbf{\mu_0} ; \mathbf{\Sigma})
        \nn\\
        &p(\mathbf{x_t} | y_t=1) \sim N(\mathbf{\mu_1} ; \mathbf{\Sigma})
	\end{align*}
In our experiment, we use the same setting for the emission models as in Unigram case, i.e. $\mathbf{\mu_0}=[-0.504,-0.264]^T$, $\mathbf{\mu_1}=[1.646,0.181]^T$ and $\mathbf{\Sigma}=0.4 \cdot \mathbf{I}$. A total of $T=60,000$ labels $\mc{Y}$ and data points $\mc{X}$ are created, and randomly split to 50,000 training, 5,000 testing and 5,000 validation partitions. 

Note that, instead of estimating the ground-truth language model  $\mathbf{P_{LM}}$ from the generated data, we directly calculate  $\mathbf{P_{LM}}$ from transition matrix $\mathbf{P_{Trans}}$ as follows. Due to Markov property, the language model ${P_{LM}}= {p(y_t,y_{t-1})}$ is actually joint probability, i.e. product of prior probability and conditional probability as: 
	\begin{align}
    {P_{LM}}=p(y_{t-1},y_t)= p(y_{t-1})p(y_{t}|y_{t-1})
    \label{Equ:Marginal}
    \end{align}
where $p(y_{t-1})$ is actually marginal probability being steady state of transition of $\mathbf{P_{Trans}}$ after infinite steps:
	\begin{align}
    \mathbf{p_{st}}= [p(y_{t-1}=0),p(y_{t-1}=1)] =\mathbf{e} \cdot \mathbf{P_{Trans}^{\infty}}
    \label{Equ:Steady}
    \end{align}
where $\mathbf{e}$ indicate unit row vector, i.e. $\mathbf{e}=(1,0) $. For our transition example in Eqn (\ref{Equ:TransExample}), one can calculate that $\mathbf{p_{st}}= [0.692,0.308]$. Thus, from Eqn. (\ref{Equ:TransExample}, \ref{Equ:Marginal}, \ref{Equ:Steady}), the Bigram language model $\mathbf{P_{LM}}$ can be calculated by:
	\begin{align}
		\mathbf{P_{LM}}
        &=
        \begin{bmatrix}
        p(y_{t-1}=0)p(y_{t}=0|y_{t-1}=0) & p(y_{t-1}=0)p(y_{t}=1|y_{t-1}=0) \\ 
        p(y_{t-1}=1)p(y_{t}=0|y_{t-1}=1) & p(y_{t-1}=1)p(y_{t}=1|y_{t-1}=1)
        \end{bmatrix}
        \nn\\
        &=diag(\mathbf{p_{st}}) \mathbf{P_{Trans}}
		\label{Equ:PLM}
	\end{align}
where $diag(\mathbf{p_{st}})$ stands for the diagonal matrix with vector $\mathbf{p_{st}}$ as its main diagonal. Using Eqn. (\ref{Equ:PLM}), one can calculate the Bigram language model for our example is:
	\begin{align*}
		&\mathbf{P_{LM}}
        =
        \begin{bmatrix}
        0.4154 & 0.2769 \\ 
        0.2769 & 0.0308
        \end{bmatrix}
    \end{align*}

\textbf{Experiment Results} To fully validate the method, we repeatedly generate 10 synthetic datasets using different HMMs stated in the previous part of the article. Each of the HMMs in the $i-th$ experiment has all the same settings except the transition matrix $\mathbf{P_{Trans}^{(i)}}$. Here list all the ${\mathbf{P_{Trans}^{(i)}}}$ we used in experiments:

\begin{align*}
        {\mathbf{P_{Trans}^{(1)}}}
        =
        \begin{bmatrix}
        0.1 & 0.9  \\ 
        0.8 & 0.2
        \end{bmatrix}
        ~~
		{\mathbf{P_{Trans}^{(2)}}}
        &=
        \begin{bmatrix}
        0.6 & 0.4  \\ 
        0.9 & 0.1
        \end{bmatrix}
        ~~
        {\mathbf{P_{Trans}^{(3)}}}
        =
        \begin{bmatrix}
        0.2 & 0.8  \\ 
        0.5 & 0.5
        \end{bmatrix}
        ~~{\mathbf{P_{Trans}^{(4)}}}
        =
        \begin{bmatrix}
        0.3 & 0.7  \\ 
        0.4 & 0.6
        \end{bmatrix}
        \nn\\
        {\mathbf{P_{Trans}^{(5)}}}
        =
        \begin{bmatrix}
        0.4 & 0.6  \\ 
        0.1 & 0.9
        \end{bmatrix}
        ~~{\mathbf{P_{Trans}^{(6)}}}
        &=
        \begin{bmatrix}
        0.5 & 0.5  \\ 
        0.7 & 0.3
        \end{bmatrix}
        ~~{\mathbf{P_{Trans}^{(7)}}}
        =
        \begin{bmatrix}
        0.7 & 0.3  \\ 
        0.8 & 0.2
        \end{bmatrix}
        ~~{\mathbf{P_{Trans}^{(8)}}}
        =
        \begin{bmatrix}
        0.8 & 0.2  \\ 
        0.4 & 0.6
        \end{bmatrix}
        \nn\\
        ~~{\mathbf{P_{Trans}^{(9)}}}
        &=
        \begin{bmatrix}
        0.5 & 0.5  \\ 
        0.4 & 0.6
        \end{bmatrix}
        ~~{\mathbf{P_{Trans}^{(10)}}}
        =
        \begin{bmatrix}
        0.9 & 0.1  \\ 
        0.3 & 0.7
        \end{bmatrix}
	\end{align*}

Then, we train the model using algorithm \ref{Alg:SPDG} with the same set of hyper-parameters in Table 1. We use training set for learning the model and validation set for early stop, test set for producing results. We repeat this process 10 times on the 10 generated dataset, and the test errors are listed in Table \ref{Tab:ResultsOfBigramExp}.

\begin{table}[t]
\caption{The classification error on test data using Bigram}
\begin{center}
\begin{tabular}{lccc}
\hline
 & Supervised & Unsupervised \\
\hline
SynData1 		&5.38\%	  	&5.55\%   \\
SynData2 		&3.62\%		&3.63\%  \\
SynData3		&2.95\% 	&2.96\%  \\
SynData4 		&2.79\% 	&2.79\%   \\
SynData5 		&1.06\%		&1.06\%    \\
SynData6 		&4.52\%		&4.54\%    \\
SynData7 		&5.66\%     &6.25\%   \\
SynData8 		&5.13\%	  	&5.66\%    \\
SynData9 		&3.89\%	  	&6.65\%    \\
SynData10 		&5.70\%	  	&5.70\%   \\
\hline
\end{tabular}
\end{center}
\label{Tab:ResultsOfBigramExp}
\end{table}

\textbf{Discussion of Experiment Results} From Table \ref{Tab:UnigramResults}, we find all the error rates in the unsupervised learning column are very low, which verify the SPDG method we introduce in the article being effective. 

Specifically, we first observe by comparing the numbers horizontally that the unsupervised SPDG method can achieve the almost the same performance as supervised learning on each dataset. We also observe that the margin between supervised and unsupervised learning for each experiment are mostly less than $1\%$. Furthermore, comparing the numbers vertically in the table, it seems that different performances are yielded by different transition matrix. However, the variance of performances over 10 experiments is actually caused by the data sampling from Gaussian models, not the label generation using $\mathbf{P_{Trans}}$, because the error rates of supervised and unsupervised for each experiment are consistent. Otherwise, the results of supervised ones should be all the same since this training method is independent from sequence property lying in the $\mathbf{P_{Trans}}$.

Recall the Unigram result we obtained previously (error rate $30.8\%$), with only changing the language model from Unigram to Bigram, the unsupervised model gains significant improvement in performance.  Therefore, we can conclude that the sequential structure is indeed an important prior to be used in supervised learning. The Bigram language model is capable of capturing this prior, despite that it is only the simplest sequential statistics prior.

\section{Conclusions \& Summary}
\label{SECTION:Conclusions}

In this article, we first surveyed major classes of unsupervised learning methods developed in the past. We then focused on a special class of unsupervised learning techniques designed for classification tasks without using labeled data in training via direct, end-to-end optimization. We motivated the methods using the prominent example of encryption technique called Caesar Cipher. In technical terms, we described a novel objective function for optimization in solving this class of unsupervised learning problems and introduced a new stochastic primal-dual gradient method for solving the optimization, accomplishing the desired unsupervised learning for classification. 

In this tutorial exposition, we thoroughly carried out the needed mathematical analysis and illustrated the step-by-step solutions to unsupervised learning for the cases of Unigram language model and Bigram language model, respectively. The readers should be able to follow the self-contained derivations we provided in this article step by step, without referring to other material. For the readers who want to implement the unsupervised learning methods without following the detailed derivation steps, we have provided a summary of the computation steps in the learning algorithm for the Unigram and Bigram cases in two separate columns in Table \ref{Tab:Summary}. 

One technical contribution of this article worth mentioning here is the data synthesis experiments, as described in Section \ref{SECTION:Experiment}, which we used to explain and evaluate the unsupervised learning algorithm. In these computational experiments, we generated a large set of synthetic sequence data, and use them to validate the mathematical analysis and derivation. We confirm that the optimization method of SPDG is capable of effectively exploiting the sequential structure of data on learning the mapping from input to output without the use of output labels.

The ideas explained in this article open up the potential for unsupervised learning that would have many practical applications involving sequential data in real world. An example is unsupervised character recognition and English spelling correction as explored in \citep{NIPS2017}. More recently, this approach is extended by \cite{Chen2019} to speech recognition where the output sequences contain both sequence and segmentation structures, creating the first success of fully unsupervised speech recognition following the earlier proposals in \citep{Chen2016,Deng2015}. Another future extension is to further exploit the structures of the input data and combine it with our method. For example, in \citep{eslami2018neural}, the authors proposed a Generative Query Network (GQN) to learn internal representations of scenes from different viewpoints and then use them to render the same scene of a different viewpoint. With such a learning process, GQN is able to learn representations that model the inherent input data regularity. Our method is orthogonal to GQN in that we can always combine our cost function with those that exploit the input structures, which could potentially lead to better unsupervised learning methods.

The exposition provided in this article has been limited to a binary sequence classification problem with label-free data in training the classifier parameters. This was made possible by making use of the property of sequential structure in the data as represented by language models. In practically all past work, sequential classification or pattern recognition problems have been tackled with supervised learning requiring labeled data, even if the sequential structure of the data was available \citep{He2008,Mesnil2013}. While the structure of the data exploited is expressed in terms of language models (both Unigram and Bigram), other ways of representing the structure of the data (e.g. graph) can also be exploited within the framework of unsupervised learning described in this tutorial, and we leave this extension to the interested readers.

\begin{sidewaystable}
\caption{Summary of Computation Steps and Comparisons of Using Unigram and Bigram}
\centering 
\renewcommand{\arraystretch}{1.8}
\begin{tabular}{|m{1cm}| >{\centering\arraybackslash}m{7.6cm}| >{\centering\arraybackslash}m{7.6cm}|}
\hline 
&Unigram & Bigram \\[1ex] 
\hline

$\mathbf{P_{LM}}$	& $\begin{array} {lcl} \begin{bmatrix} p_{0},~~p_{1}\end{bmatrix} \end{array}$ (Eqn. \ref{Equ:prior})& $\begin{bmatrix}
        p_{00} & p_{01} \\ 
        p_{10} & p_{11}
        \end{bmatrix}$ (Eqn. \ref{Equ:Bigram_PLM}) \\ [1ex]
\hline

$\mathbf{\overline{P_{LM}}} (\theta)$	& $\begin{bmatrix}
        \frac{1}{T}\sum_{t=1}^T p_{\theta,t}(0),~~\frac{1}{T} \sum_{t=1}^T p_{\theta,t}(1)
        \end{bmatrix}^T$ (Eqn. \ref{Equ:classifierstat})& $\frac{1}{T}\sum_{t=1}^T \mathbf{P_{t-1,\theta}} \mathbf{P_{t,\theta}}^T$ (Eqn. \ref{Equ:classifierstat})\\ [1ex]

\hline 
$\mc{J}(\theta)$	& $-\mathbf{P_{LM}}  \ln{\mathbf{\overline{P_{LM}}} (\theta)}$ (Eqn. \ref{Equ:CostUnigram})&  $-\left \langle \mathbf{P_{LM}} ~, ~\ln{\mathbf{\overline{P_{LM}}} (\theta)} \right \rangle$ (Eqn. \ref{Equ:Cost})\\ [1ex]

\hline 
$\mc{L}(\theta,\mathbf{V})$	
& $\!\begin{aligned}[t][\mathbf{P_{LM}} \odot  \mathbf{V}]  \mathbf{\overline{P_{LM}}} (\theta)  
   \\ +\mathbf{P_{LM}} \ln{(-\mathbf{V})}^T \end{aligned}$ ~~~~~(Eqn. \ref{Equ:finalMatrix})
& $\!\begin{aligned}[t] \frac{1}{T} \sum_{t=1}^{T} (\mathbf{p_{t-1,\theta}^T}   [\mathbf{P_{LM}} \odot  \mathbf{V}]  \mathbf{p_{t,\theta}} )   
     \\+\left \langle \mathbf{P_{LM}}, \ln{(-\mathbf{V})} \right \rangle \end{aligned}$ ~~~~~~(Eqn. \ref{Equ:finalMat})\\ [1ex]

\hline 
$\frac{\partial \mc{L}}{\partial \mathbf{\theta}}$	
& $\frac{1}{T} \sum_{t=1}^{T} \bigg \{
\gamma p_{t-1,\theta}(0) p_{t-1,\theta}(1)   [\mathbf{P_{LM}} \odot \mathbf{V}] \mathbf{X} \bigg \}$ (Eqn. \ref{Equ:finalDiv_theta})
&  $\!\begin{aligned}[t]\frac{1}{T} \sum_{t=1}^{T} \bigg \{
		\gamma P_{t-1,\theta}(0) P_{t-1,\theta}(1) 
		(
	[ \mathbf{X_{t-1}^T}  (\mathbf{P_{LM}} \odot \mathbf{V}) \mathbf{P_{t,\theta}}
        \\ + (\mathbf{P_{LM}} \odot \mathbf{V})^T    \mathbf{X_{t}^T} \mathbf{P_{t-1,\theta}} ]
        \bigg \} \end{aligned}$ ~~~~~(Eqn. \ref{Equ:finalUpdateBigram})\\ [1ex]

\hline 
$\frac{\partial \mc{L}}{\partial \mathbf{V}}$	& $\frac{\partial}{\partial \mathbf{V}}
    \bigg \{
    [\mathbf{P_{LM}} \odot  \mathbf{V}]  \mathbf{\overline{P_{LM}}} (\theta)  
    +\mathbf{P_{LM}} \ln{(-\mathbf{V})}^T  
    \bigg \}$ (Eqn. \ref{Equ:finalDiv_l})&  $\frac{1}{T} \sum_{t=1}^{T} \bigg \{
        \mathbf{P_{LM}} \cdot (\mathbf{P_{t-1,\theta} P_{t,\theta}^T}) + \frac{\mathbf{P_{LM}}}{\mathbf{V}}
        \bigg \}$ (Eqn. \ref{Equ:finalUpdateBigram})\\ [1ex]

\hline 
$\mathbf{V^0}$	& $-\frac{1}{\mathbf{\overline{P_{LM}}}^T (\theta^0)} $ (Eqn. \ref{Equ:MinimaUnigram})& $-\frac{T}{\sum_{t=1}^T \mathbf{P_{t-1,\theta^0}} \mathbf{P_{t,\theta^0}^T}}$ (Eqn. \ref{Equ:MinimaBigram})\\ [1ex]


\hline 
\end{tabular}
\label{Tab:Summary}
\end{sidewaystable}





\newpage
{
\bibliography{neurip_2019}
\bibliographystyle{icml2017}
}

\end{document}